%% file: main.tex
\documentclass{article}

\usepackage{microtype}
\usepackage{graphicx}
\usepackage{booktabs} 

\usepackage{hyperref}

\usepackage[accepted]{icml2018}

\icmltitlerunning{Temporal Poisson Square Root Graphical Models}

\usepackage{amsmath,amssymb,amsthm}

\newcommand{\curly}[1]{\left\{#1\right\}}
\newcommand{\norm}[1]{\lVert #1 \rVert}
\newcommand{\Norm}[1]{\left\lVert #1 \right\rVert}

\usepackage{color}

\usepackage{enumitem}

\usepackage{bm}

\newcommand{\bb}{\mathbf{b}}

\newcommand{\be}{\mathbf{e}}

\newcommand{\bo}{\mathbf{o}}

\newcommand{\bs}{\mathbf{s}}
\newcommand{\bt}{\mathbf{t}}
\newcommand{\bw}{\mathbf{w}}
\newcommand{\bx}{\mathbf{x}}
\newcommand{\bH}{\mathbf{H}}

\newcommand{\bW}{\mathbf{W}}
\newcommand{\bX}{\mathbf{X}}
\newcommand{\bZ}{\mathbf{Z}}
\newcommand{\bzero}{\mathbf{0}}

\newcommand{\bphi}{\bm{\phi}}

\newcommand{\bDelta}{\bm{\Delta}}
\newcommand{\bTheta}{\bm{\Theta}}

\newcommand{\btheta}{\bm{\theta}}
\newcommand{\bomega}{\bm{\omega}}
\newcommand{\balpha}{\bm{\alpha}}

\newcommand{\grad}{\bm{\nabla}}

\newtheorem{theorem}{Theorem}
\newtheorem{lemma}{Lemma}
\newtheorem{assumption}{Assumption}

\usepackage{mathtools}
\DeclarePairedDelimiter\abs{\lvert}{\rvert}%
\makeatletter
\let\oldabs\abs
\def\abs{\@ifstar{\oldabs}{\oldabs*}}

\makeatother

\newcommand\given[1][]{\:#1\vert\:}

\usepackage{breqn}

\usepackage{caption}
\usepackage{subcaption}

\begin{document}

\twocolumn[
\icmltitle{Temporal Poisson Square Root Graphical Models}



\icmlsetsymbol{equal}{*}

\begin{icmlauthorlist}
\icmlauthor{Sinong Geng*}{wisc}
\icmlauthor{Zhaobin Kuang*}{wisc}
\icmlauthor{Peggy Peissig}{mcrf}
\icmlauthor{David Page}{wisc}
\end{icmlauthorlist}

\icmlaffiliation{wisc}{The University of Wisconsin, Madison}
\icmlaffiliation{mcrf}{Marshfield Clinic Research Institute}

\icmlcorrespondingauthor{Sinong Geng}{sgeng2@wisc.edu}

\icmlkeywords{Machine Learning, ICML}

\vskip 0.3in
]



\printAffiliationsAndNotice{
Sinong Geng and Zhaobin Kuang contribute equally. Their names are listed in alphabetical order.
}  

\begin{abstract}
We propose temporal Poisson square root graphical models (TPSQRs), a generalization of Poisson square root graphical models (PSQRs) specifically designed for modeling longitudinal event data. By estimating the temporal relationships for all possible pairs of event types, TPSQRs can offer a holistic perspective about whether the occurrences of any given event type could excite or inhibit any other type. A TPSQR is learned by estimating a collection of interrelated PSQRs that share the same template parameterization. These PSQRs are estimated jointly in a pseudo-likelihood fashion, where Poisson pseudo-likelihood is used to approximate the original more computationally-intensive pseudo-likelihood problem stemming from PSQRs. Theoretically, we demonstrate that under mild assumptions, the Poisson pseudo-likelihood approximation is \emph{sparsistent} for recovering the underlying PSQR. Empirically, we learn TPSQRs from Marshfield Clinic electronic health records (EHRs) with millions of drug prescription and condition diagnosis events, for adverse drug reaction (ADR) detection. Experimental results demonstrate that the learned TPSQRs can recover ADR signals from the EHR effectively and efficiently.
\end{abstract}

\section{Introduction}
\label{sec:intro}
Longitudinal event data (LED) and the analytics challenges therein are ubiquitous now. In business analytics, purchasing events of different items from millions of customers are collected, and retailers are interested in how a distinct market action or the sales of one particular type of item could boost or hinder the sales of another type \citep{han2011data}. In search analytics, web search keywords from billions of web users are usually mapped into various topics (e.g.~travel, education, weather), and search engine providers are interested in the interplay among these search topics for a better understanding  of user preferences \citep{gunawardana2011model}. In health analytics, electronic health records (EHRs) contain clinical encounter events from millions of patients collected over decades, including drug prescriptions, biomarkers, and  condition diagnoses, among others. Unraveling the relationships between different drugs and different conditions is vital to answering some of the most pressing medical and scientific questions such as drug-drug interaction detection \citep{tatonetti2012data}, comorbidity identification, adverse drug reaction (ADR) discovery \citep{simpson2013multiple, bao2017hawkes, kuang2017pharmacovigilance}, computational drug repositioning \citep{kuang2016baseline, kuang2016computational}, and precision medicine \citep{liu2013genetic, liu2014new}. 

All these analytics challenges beg the statistical modeling question: \emph{can we offer a comprehensive perspective about the relationships between the occurrences of all possible pairs of event types in longitudinal event data?} In this paper, we propose a solution via temporal Poisson square root graphical models (TPSQRs), a generalization of Poisson square root graphical models (PSQRs, \citealt{inouye2016square})  made in order to represent multivariate distributions among count variables evolving temporally in LED. 

The reason why conventional undirected graphical models (UGMs) are not readily applicable to LED is the lack of mechanisms to address the \emph{temporality} and \emph{irregularity} in the data. Conventional UGMs \citep{liu2013bayesian, liu2014learning, yang2015graphical, liu2016multiple, kuang2017screening, geng2018stochastic} focus on estimating the co-occurrence relationships among various variables rather than their temporal relationships, that is, how the occurrence of one type of event may affect the future occurrence of another type. Furthermore, existing temporal variants of UGMs \citep{kolar2010estimating, yang2015fused} usually assume that data are regularly sampled, and observations for all variables are available at each time point. Neither assumption is true, due to the irregularity of LED. 

In contrast to these existing UGM models, a TPSQR models temporal relationships. First, by \emph{data aggregation}, a TPSQR extracts a sequence of time-stamped summary count statistics of distinct event types that preserves the relative temporal order in the raw data for each subject. A PSQR is then used to model the joint distribution among these summary count statistics for each subject. Different PSQRs for different subjects are assumed to share the same \emph{template parameterization} and hence can be learned jointly by estimating the template in a pseudo-likelihood fashion. To address the challenge in temporal irregularity, we compute the exact time difference between each pair of time-stamped summary statistics, and decide whether a difference falls into a particular predefined time interval, hence transforming the irregular time differences into regular timespans. We then incorporate the effects of various timespans into the template parameterization as well as PSQR constructions from the template.

By addressing temporality and irregularity of LED in this fashion, TPSQR is also different from many point process models \citep{gunawardana2011model, weiss2012multiplicative, weiss2013forest, du2016recurrent}, which usually strive to pinpoint the exact occurrence times of events, and offer  generative mechanisms to event trajectories. TPSQR, on the other hand, adopts a coarse resolution approach to temporal modeling via the aforementioned data aggregation and time interval construction. As a result, TPSQR focuses on estimating stable relationships among occurrences of different event types, and does not model the precise event occurrence timing. This behavior is especially meaningful in application settings such as ADR discovery, where the importance of identifying the occurrence of an adverse condition caused by the prescription of a drug usually outweighs knowing about the exact time point of the occurrence of the ADR, due to the high variance of the onset time of ADRs \citep{schuemie2016detecting}.

Since TPSQR is a generalization of PSQR, many desirable properties of PSQR are inherited by TPSQR. For example, TPSQR, like PSQR, is capable of modeling both positive and negative dependencies between covariates. Such flexibility cannot usually be taken for granted when modeling a multivariate distribution over count data due to the potential dispersion of the partition function of a graphical model \citep{yang2015graphical}. TPSQR can be learned by solving the pseudo-likelihood problem for PSQR.  For efficiency and scalability, we use Poisson pseudo-likelihood to approximately solve the original pseudo-likelihood problem induced by a PSQR, and we show that the Poisson pseudo-likelihood approximation can recover the structure of the underlying PSQR under mild assumptions. Finally, we demonstrate the utility of TPSQRs using Marshfield Clinic EHRs with millions of drug prescription and condition diagnosis events for the task of adverse drug reaction (ADR) detection. Our contributions are three-fold:

\begin{itemize}[leftmargin=*]
\item TPSQR	is a generalization of PSQR made in order to represent the multivariate distributions among count variables evolving temporally in LED. TPSQR can accommodate both positive and negative dependencies among covariates, and can be learned efficiently via the pseudo-likelihood problem for PSQR.

\item In terms of advancing the state-of-the-art of PSQR estimation, we propose Poisson pseudo-likelihood approximation in lieu of the original more computationally-intensive conditional distribution induced by the joint distribution of a PSQR. We show that under mild assumptions, the Poisson pseudo-likelihood approximation procedure is sparsistent \citep{ravikumar2007spam} with respect to the underlying PSQR. Our theoretical results not only justify the use of the more efficient Poisson pseudo-likelihood over the original conditional distribution for better estimation efficiency of PSQR but also establish a formal correspondence between the more intuitive but less stringent local Poisson graphical models \citep{allen2013local} and the more rigorous but less convenient PSQRs.

\item  We apply TPSQR to Marshfield Clinic EHRs to determine the relationships between the occurrences of various drugs and the occurrences of various conditions, and offer more accurate estimations for adverse drug reaction (ADR) discovery, a challenging task in health analytics due to the (thankfully) rare and weak ADR signals encoded in the data, whose success is crucial to improving healthcare both financially and clinically \citep{sultana2013clinical}. 
\end{itemize}


\section{Background}
\label{sec:backgrond}

We show how to deal with the challenges in temporality and irregularity mentioned in Section~\ref{sec:intro} via the use of data aggregation and an influence function for LED. We then define the template parameterization that is central to the modeling of TPSQRs. 

\subsection{Longitudinal Event Data}
\label{sec:led}
Longitudinal event data are time-stamped events of finitely many types collected across various subjects over time. Figure~\ref{fig:led} visualizes the LED for two subjects. As shown in Figure~\ref{fig:led}, the occurrences of different event types are represented as arrows in different colors. No two events for one subject occur at the exact same time. We are interested in modeling the relationships among the occurrences of different types of events via TPSQR. 

\begin{figure*}[t]
\centering
\includegraphics[scale=1.2]{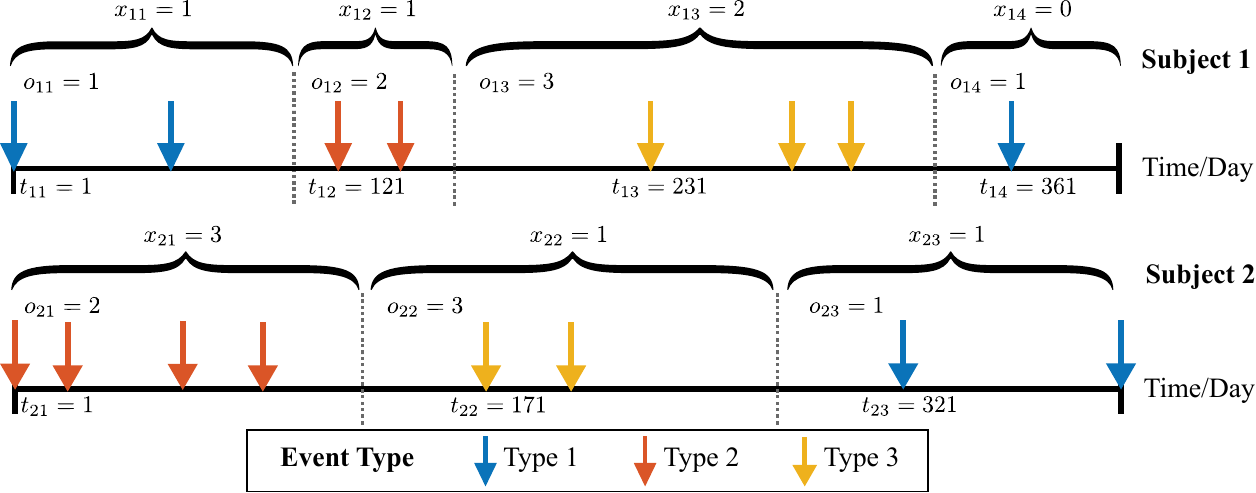}
\caption{Visualization of longitudinal event data from two subjects. Curly brackets denote the timespans during which events of only one type occur.  $x_{ij}$'s represent the number of subsequent occurrences after the first occurrence.\comment{ hence $x_{ij}$'s are always one less than the total number of arrows contained in each timespan.} $o_{ij}$'s are the types of events in various timespans.}
\label{fig:led}
\end{figure*}

\subsection{Data Aggregation}
\label{sec:data-aggreation}
To enable PSQRs to cope with the temporality in LED, TPSQRs start from extracting relative-temporal-order-preserved summary count statistics from the raw LED via data aggregation, to cope with the high volume and frequent consecutive replications of events of the same type that are commonly observed in LED. Take Subject 1 in Figure~\ref{fig:led} as an illustrative example; we divide the raw data of Subject 1 into four timespans by the dashed lines. Each of the four timespans contains only events of the same type. We use three statistics to summarize each timespan: the time stamp of the first occurrence of the event in each timespan: $t_{11}=1$, $t_{12}=121$, $t_{13}=231$, and $t_{14}=361$; the event type in each timespan: $o_{11}=1$, $o_{12}=2$, $o_{13}=3$, and $o_{14}=1$; and the counts of subsequent occurrences in each timespan: $x_{11}=1$, $x_{12}=1$, $x_{13}=2$, and $x_{14}=0$. Note that the reason  $x_{14}=0$ is that there is only one occurrence of event type $1$ in total during timespan $4$ of subject $1$. Therefore, the number of subsequent occurrence after the first and only occurrence is $0$.

Let there be $N$ independent subjects and $p$ types of events in a given LED $\mathbb{X}$. We denote by $n_i$ the number of timespans during which only one type of event occurs to subject $i$, where $i \in \curly{1,2,\cdots,N}$.  The $j^{th}$ timespan of the $i^{th}$ subject can be represented by the vector $\bs_{ij}:=\begin{bmatrix} t_{ij} & o_{ij} & x_{ij} \end{bmatrix}^\top$, where $j \in \curly{1,2,\cdots, n_i}$, and ``$:=$'' represents ``defined as." $t_{ij} \in [0,+\infty)$ is the time stamp at which the first event occurs during the timespan $\bs_{ij}$. Furthermore, $t_{11} < t_{12} < \cdots < t_{1n_i}$. $o_{ij} \in \curly{1,2,\cdots,p}$ represents the event type in $\bs_{ij}$. $o_{ij}\ne o_{i(j+1)}$, $\forall i \in \curly{1,2,\cdots,N}$ and $\forall j<n_i$. $x_{ij} \in \mathbb{N}$ is the number of subsequent occurrences of events of the same type in $\bs_{ij}$.

\subsection{Influence Function}
Let $\bs_{ij}$ and $\bs_{ij'}$ be given, where $j < j' \le n_i$. To handle the irregularity of the data, we map the time difference $t_{ij'} - t_{ij}$ to a one-hot vector that represents the activation of a time interval using an \emph{influence function} $\bphi(\cdot)$, a common mechanism widely used in point process models and signal processing. In detail, let $L+1$ user-specified time-threshold values be given, where $0=\tau_0 < \tau_1 < \tau_2 < \cdots < \tau_L$. $\bphi(\tau)$ is a $L \times 1$ one hot vector whose $l^{th}$ component is defined as:
\begin{equation}
\label{eq:phi}
[\bphi(\tau)]_l := \begin{cases}
1, &  \tau_{l-1} \le \tau < \tau_l\\
0, &  \text{otherwise}
\end{cases},
\end{equation}
where $l \in \curly{1,2,\cdots,L}$. In our case, we let $\tau := t_{ij'}-t_{ij}$ to construct $\bphi(\tau)$ according to \eqref{eq:phi}. Widely used influence functions in signal processing include the dyadic wavelet function and the Haar wavelet function \citep{mallat2008wavelet}; both are piecewise constant and hence share similar representation to \eqref{eq:phi}.


\subsection{Template Parameterization}
\label{sec:template-para}
Template parameterization provides the capability of TPSQRs to represent the effects of all possible (ordered) pairs of event types on all time scales. Specifically, let an ordered pair $(k,k') \in \curly{1,2,\cdots,p}^2$ be given. Let $0=\tau_0 < \tau_1 < \tau_2 < \cdots < \tau_K$ also be given. For the ease of presentation, we assume that $k \neq k'$, which can be easily generalized to $k = k'$. Considering a particular patient, we are interested in knowing the effect of an occurrence of a type $k$ event towards a subsequent occurrence of a type $k'$ event, when the time between the two occurrences falls in the $l^{th}$ time window specified via \eqref{eq:phi}. Enumerating all $L$ time windows, we have:
\begin{equation}
\label{eq:w}
\bw_{kk'} := \begin{bmatrix}
w_{kk'1} & w_{kk'2} & \cdots & w_{kk'L}
\end{bmatrix}^\top.
\end{equation}
Note that since $(k,k')$ is ordered, $\bw_{k'k}$ is different from $\bw_{kk'}$. We further define $\bW$ as a $(p-1)p \times L$ matrix that stacks up all $\bw_{kk'}^\top$'s. In this way, $\bW$ includes all possible pairwise temporally bidirectional relationships among the $p$ variables on different time scales, offering holistic representation power. To represent the intrinsic prevalence effect of the occurrences of events of various types, we further define $\bomega := \begin{bmatrix} \omega_1 & \omega_2 & \cdots & \omega_p
\end{bmatrix}^\top$. We call $\bomega$ and $\bW$ the template parameterization, from which we will generate the parameters of various PSQRs as shown in Section~\ref{sec:model}.

\section{Modeling}
\label{sec:model}
Let $\bs_{ij}$'s be given where $j \in \curly{1,2,\cdots,n_i}$; we demonstrate the use of the influence function and template parameterization to construct a PSQR for subject $i$.  

Let $\bt_i := \begin{bmatrix}
t_{i1} \ t_{i2} \ \cdots \ t_{in_i}
\end{bmatrix}^\top$, $\bo_i := \begin{bmatrix}
o_{i1} \ o_{i2} \ \cdots \ o_{in_i}
\end{bmatrix}^\top$, and $\bx_i := \begin{bmatrix}
x_{i1} & x_{i2} & \cdots & x_{in_i}
\end{bmatrix}^\top$. Given $\bt_i$ and $\bo_i$, a TPSQR aims at modeling the joint distribution of counts $\bx_i$ using a PSQR. Specifically, under the template parameterization $\bomega$ and $\bW$, we first define a symmetric parameterization $\bTheta^{(i)}$ using $\bt_i$ and $\bo_i$. The component of $\bTheta^{(i)}$ at the $j^{th}$ row and the ${j'}^{th}$ column is:
\begin{equation}
\label{eq:omega-i}
\theta_{jj'}^{(i)}:= [\bTheta^{(i)}]_{jj'} 
\hspace{-1mm} := \hspace{-1mm} \begin{cases}
\omega_{o_{ij}}, & \hspace{-2mm} j=j'\\
\bw_{o_{ij}o_{ij'}}^\top \bphi (\lvert t_{ij'} -t_{ij} \rvert), & \hspace{-2mm} j < j'\\
[\bTheta^{(i)}]_{j'j}, & \hspace{-2mm} j>j'
\end{cases}.
\end{equation}
We then can use $\bTheta^{(i)}$ to parameterize a PSQR that gives a joint distribution over $\bx_i$ as:
\begin{align}
\label{eq:psqr-joint-prob}
\begin{split} 
\text{P} \Big(\bx_i; \bTheta^{(i)} \Big) \hspace{-1mm} := & \exp \hspace{-1mm} \Bigg[ \sum_{j=1}^{n_i} \theta_{jj}^{(i)} \sqrt{x_{ij}} + \sum_{j=1}^{n_i-1}\sum_{j'>j}^{n_i} \theta_{jj'}^{(i)} \sqrt{x_{ij}x_{ij'}}\\
- & \sum_{j=1}^{n_i} \log (x_{ij}!) - A_{n_i}\left(\bTheta^{(i)}\right) \Bigg].
\end{split}
\end{align}
In \eqref{eq:psqr-joint-prob},  $A_{n_i}(\bTheta^{(i)})$ is a normalization constant called the log-partition function that ensures the legitimacy of the probability distribution in question:
\begin{align}
\label{eq:A}
\begin{split}
A_{n_i}&\left(\bTheta^{(i)}\right):= \log \hspace{-2mm} \sum_{\bx \in \mathbb{N}^{n_i}}\hspace{-2mm} \exp \Bigg[ \sum_{j=1}^{n_i} \theta_{jj}^{(i)} \sqrt{x_j}\\
& + \sum_{j=1}^{n_i-1}\sum_{j'>j}^{n_i} \theta_{jj'}^{(i)} \sqrt{x_{j}x_{j'}} - \sum_{j=1}^{n_i} \log (x_{j}!) \Bigg].
\end{split}
\end{align}
Note that in \eqref{eq:A} we emphasize the dependency of the partition function upon the dimension of $\bx$ using the subscript $n_i$, and $\bx := \begin{bmatrix}
x_1 & x_1 & \cdots & x_{n_i} \end{bmatrix}$.

To model the joint distribution of $\bx_i$, TPSQR directly uses $\bTheta^{(i)}$, which is extracted from $\bomega$ and $\bW$ via \eqref{eq:omega-i} depending on the individual and temporal irregularity of the data characterized by $\bt_i$ and $\bo_i$. Therefore, $\bomega$ and $\bW$ serve as a template for constructing $\bTheta^{(i)}$'s, and hence provide a ``template parameterization.'' Since there are $N$ subjects in total in the dataset, and each $\bTheta^{(i)}$ offers a personalized PSQR for one subject, TPSQR is capable of learning a collection of interrelated PSQRs due to the use of the template parameterization. Recall the well-rounded representation power of a template shown in Section~\ref{sec:template-para}; learning the template parameterization via TPSQR can hence offer a comprehensive perspective about the relationships for all possible temporally ordered pairs of event types.

Furthermore, since TPSQR is a generalization of PSQR, it inherits many desirable properties enjoyed by PSQR. A most prominent property is its capability of accommodating both positive and negative dependencies between variables. Such flexibility in general cannot be taken for granted when modeling multivariate count data. For example, a Poisson graphical model \citep{yang2015graphical} can only represent negative dependencies due to the diffusion of its log-partition function when positive dependencies are involved. Yet for example one drug (e.g., the blood thinner Warfarin) can have a positive influence on some conditions (e.g., bleeding) and a  negative influence on others (e.g., stroke). We refer interested readers to \citealt{allen2013local, yang2013poisson, inouye2015fixed, yang2015graphical, inouye2016square} for more details of PSQRs and other related Poisson graphical models.

%

\section{Estimation}
\label{sec:estimation}

In this section, we present the pseudo-likelihood estimation problem for TPSQR. We then point out that solving this problem can be inefficient, which leads to the proposed Poisson pseudo-likelihood approximation to the original pseudo-likelihood problem.

\subsection{Pseudo-Likelihood for TPSQR}
\label{sec:pl-tpsqr}
We now present our estimation approach for TPSQR based on pseudo-likelihood. We start from considering the pseudo-likelihood for a given $i^{th}$ subject. By \eqref{eq:psqr-joint-prob}, the log probability of $x_{ij}$ conditioned on $\bx_{i,-j}$, which is an $(n_i-1)\times 1$ vector constructed by removing the $j^{th}$ component from $\bx_i$, is given as:
\allowdisplaybreaks
\begin{align}
\label{eq:cond-dist}
\begin{split}
\log & \text{P} \bigg( x_{ij} \given \bx_{i,-j}; \btheta^{(i)}_j\bigg) =  - \log (x_{ij}!) +\\
\bigg( & \theta_{jj}^{(i)} + {\btheta_{j,-j}^{(i)\top}} \sqrt{\bx_{i,-j}} \bigg) \sqrt{x_{ij}} 
 - \tilde{A}_{n_i}\left(\btheta^{(i)}_j\right),
\end{split}
\end{align}
where $\btheta_j^{(i)}$ is the $j^{th}$ column of $\bTheta^{(i)}$ and hence
\begin{align}
\btheta_j^{(i)} :=& \begin{bmatrix}
\theta_{1j}^{(i)} \ \  \cdots \ \ \theta_{j-1,j}^{(i)} \quad  \theta_{jj}^{(i)} \quad \theta_{j+1,j}^{(i)} \ \ \cdots \ \ \theta_{n_i,j}^{(i)}
\end{bmatrix}^\top \nonumber \\
\label{eq:theta-row-j}
:= & \begin{bmatrix}
\theta_{1j}^{(i)} \ \  \cdots \ \ \theta_{j-1,j}^{(i)} \quad  \theta_{jj}^{(i)} \quad \theta_{j,j+1}^{(i)} \ \ \cdots \ \ \theta_{j,n_i}^{(i)}
\end{bmatrix}^\top \hspace{-2mm}.
\end{align}
In \eqref{eq:theta-row-j}, by the symmetry of $\bTheta^{(i)}$, we rearrange the index after $\theta_{jj}^{(i)}$ to ensure that the row index is no larger than the column index so that the parameterization is consistent with that in \eqref{eq:psqr-joint-prob}. We will adhere to this convention in the subsequent presentation. Furthermore, $\btheta_{j,-j}^{(i)}$ is an $(n_i-1)\times 1$ vector constructed from $\btheta_j^{(i)}$ by excluding its $j^{th}$ component, and $\sqrt{\bx_{i,-j}}$ is constructed by taking the square root of each component of $\bx_{i,-j}$. Finally,
\begin{align}
\label{eq:cond-partition}
\begin{split}
\tilde{A}&_{n_i}\left(\btheta^{(i)}_j\right) := \\
&\log \hspace{-2mm}\sum_{\bx \in \mathbb{R}^{n_i}}\hspace{-2mm} \exp \hspace{-1mm}\left[\bigg(  \theta_{jj}^{(i)} + {\btheta_{j,-j}^{(i)\top}} \sqrt{\bx_{i,-j}} \bigg) \sqrt{x_{ij}} 
 - \log (x_{ij}!)\right],
\end{split}
\end{align}
which is a quantity that involves summing up infinitely many terms, and in general cannot be further simplified, leading to potential intractability in computing \eqref{eq:cond-partition}. 

With the conditional distribution in \eqref{eq:cond-dist} and letting $M := \sum_{i=1}^N n_i$, the pseudo-likelihood problem for TPSQR is given as:
\begin{equation}
\label{eq:pseudo-obj}
\max_{\bomega,\bW} \frac{1}{M} \sum_{i=1}^N \sum_{j=1}^{n_i} \log \text{P} \left( x_{ij} \given \bx_{i,-j}; \btheta^{(i)}_j\right).
\end{equation}
\eqref{eq:pseudo-obj} is the maximization over all the conditional distributions of all the count variables for all $N$ personalized PSQRs generated by the template. \emph{Therefore, it can be viewed as a pseudo-likelihood estimation problem directly for $\bomega$ and $\bW$.} However, solving the pseudo-likelihood problem in \eqref{eq:pseudo-obj} involves the predicament of computing the potentially intractable \eqref{eq:cond-partition}, which motivates us to use Poisson pseudo-likelihood as an approximation to \eqref{eq:pseudo-obj}.

\subsection{Poisson Pseudo-Likelihood}

Using the parameter vector $\btheta_j^{(i)}$, we define the conditional distribution of $x_{ij}$ given by $\bx_{i,-j}$ via the Poisson distribution as:
\begin{align}
\label{eq:cond-dist-poisson}
\begin{split}
\hat{\text{P}}&\left( x_{ij} \given \bx_{i,-j};  \btheta^{(i)}_j \right) \propto 
\frac{\exp\hspace{-1mm} \left[ \left(\hspace{-1mm} \theta_{jj}^{(i)} \hspace{-1mm} + {\btheta_{-j}^{(i)\top}} \hspace{-1mm} \bx_{i,-j} \hspace{-1mm}\right)  x_{ij} \right]} {x_{ij}!}.
\end{split}
\end{align}
Notice the similarity between \eqref{eq:cond-dist} and \eqref{eq:cond-dist-poisson}. We can define the sparse Poisson pseudo-likelihood problem similar to the original pseudo-likelihood problem by replacing $\log \text{P} \left(x_{ij} \given \bx_{i,-j};  \btheta^{(i)}_j \right)$ with $\log \hat{\text{P}} \left(x_{ij} \given \bx_{i,-j};  \btheta^{(i)}_j \right)$  :
\begin{equation}
\label{eq:pseudo-poisson}
\max_{\bomega,\bW} \frac{1}{M} \hspace{-1mm} \sum_{i=1}^N \sum_{j=1}^{n_i} \log \hat{\text{P}} \left( \hspace{-0.5mm} x_{ij} \given \bx_{i,-j}; \btheta^{(i)}_j\hspace{-0.5mm}\right) - \lambda \norm{\bW}_{1,1},
\end{equation}
where $\lambda \ge 0$ is the regularization parameter, and the penalty
\begin{equation*}
\norm{\bW}_{1,1} :=  \sum_{i=1}^{(p-1)p}\sum_{j=1}^L \left\lvert [\bW]_{ij} \right\rvert
\end{equation*}
is used to encourage sparsity over the template parameterization $\bW$ that determines the interactions between the occurrences of two distinct event types. As mentioned at the end of Section~\ref{sec:pl-tpsqr}, TPSQR learning is equivalent to learning a PSQR over the template parameterization. Therefore, the sparsity penalty induced here is helpful to recover the structure of the underlying graphical model.

The major advantage of approximating the original pseudo-likelihood problem with Poisson pseudo-likelihood is the gain in computational efficiency. Based on the construction in \citealt{geng2017efficient}, \eqref{eq:pseudo-poisson} can be formulated as an $l_1$-regularized Poisson regression problem, which can be solved much more efficiently via many sophisticated algorithms and their implementations \citep{friedman2010regularization, tibshirani2012strong} compared to solving the original problem that involves the potentially challenging computation  for \eqref{eq:cond-partition}. Furthermore, in the subsequent section, we will show that even though the Poisson pseudo-likelihood is an approximation procedure to the pseudo-likelihood of PSQR, under mild assumptions Poisson pseudo-likelihood is still capable of recovering the structure of the underlying PSQR.

\subsection{Sparsistency Guarantee}
\label{sec:theory}
For the ease of presentation, in this section we will reuse much of the notation that appears previously. The redefinitions introduced in this section only apply to the contents in this section and the related proofs in the Appendix. Recall at the end of Section~\ref{sec:pl-tpsqr}, the pseudo-likelihood problem of TPSQR can be viewed as learning a PSQR parameterized by the template. Therefore, without loss of generality, we will consider a PSQR over $p$ count variables $\bX = \bx \in \mathbb{N}^p$ parameterized by a $p \times p$ symmetric matrix $\bTheta^*$, where $\bX:=\begin{bmatrix}
X_1 & X_2 & \cdots & X_p \end{bmatrix}^\top$ is the multivariate random variable, and $\bx$ is an assignment to $\bX$. We use $\norm{\cdot}_{\infty}$ to represent the infinity norm of a vector or a matrix. Let $\mathbb{X} := \curly{ \bx_1,\bx_2,\cdots,\bx_n}$ be a dataset with $n$  independent and identically distributed (i.i.d.) samples generated from the PSQR. Then the joint probability distribution over $\bx$ is:
\begin{align*}
\begin{split}
\text{P} (\bx; \bTheta^* ) \hspace{-1mm} := \exp \hspace{-1mm} \Bigg[ & \sum_{j=1}^p \theta_{jj}^* \sqrt{x_{ij}} + \sum_{j=1}^{p-1}\sum_{j'>j}^p \theta_{jj'}^* \sqrt{x_{ij}x_{ij'}} 
\\- & \sum_{j=1}^{p} \log (x_{ij}!) - A\left(\bTheta^*\right) \Bigg],
\end{split}
\end{align*}
where $A\left(\bTheta^*\right)$ is the log-partition function, and the corresponding Poisson pseudo-likelihood problem is:
\begin{equation}
\label{eq:pre}
\hat{\bTheta} := \arg\min_\bTheta F(\bTheta) + \lambda\norm{\bTheta}_{1,\text{off}},
\end{equation}
where
\begin{align}
\label{eq:F}
\begin{split}
F(\bTheta):=\frac{1}{n}\sum_{i=1}^n \sum_{j=1}^p \Big[-& (\theta_{jj} + \btheta_{j,-j}^\top \bx_{i,-j} )x_{ij} \\
+ & \exp (\theta_{jj} + \btheta_{j,-j}^\top \bx_{i,-j})  \Big],
\end{split}
\end{align}
and $\norm{\bTheta}_{1,\text{off}}$ represents imposing $l_1$ penalty over all but the diagonal components of $\bTheta$. 

\emph{Sparsistency \citep{ravikumar2007spam} addresses whether $\hat{\bTheta}$ can recover the structure of the underlying $\bTheta^*$ with high probability using $n$ i.i.d.~samples.} In what follows, we will show that $\hat{\bTheta}$ is indeed sparsistent under mild assumptions. 

We use $\mathbb{E}[\cdot]$ to denote the expectation of a random variable under $\text{P} (\bx; \bTheta^* )$.
The first assumption is about the boundedness of $\mathbb{E}[\bX]$, and the boundedness of the partial second order derivatives of a quantity related to the log-partition $A\left(\bTheta^*\right)$. This assumption is standard in the analysis of pseudo-likelihood methods \citep{yang2015graphical}.

\begin{assumption}
\label{asm:1}
$\norm{\mathbb{E}[\bX]}_{\infty} \le C_1$ for some $C_1 > 0$. Let
\begin{align*}
B(\bTheta,\bb) := \log \hspace{-1mm} & \sum_{\bx \in \mathbb{N}^p}\hspace{-1mm} \exp \hspace{-1mm} \Bigg[ \sum_{j=1}^p \theta_{jj} \sqrt{x_j} + {\bb}^\top \bx \\
+ & \sum_{j=1}^{p-1}\sum_{j'>j}^p \theta_{jj'} \sqrt{x_{j}x_{j'}} - \sum_{j=1}^p \log (x_{j}!) \Bigg].
\end{align*}
For some $C_2>0$, and $\forall k\in[0,1]$,
\begin{equation*}
\forall j \in \curly{1,2,\cdots,p},\quad \frac{\partial^2 B(\bTheta,\bzero+k \be_j)}{\partial^2 b_j} \le C_2,
\end{equation*}
where $\be_j$ is the one-hot vector with the $j^{th}$ component as $1$.
\end{assumption}

The following assumption characterizes the boundedness of the conditional distributions given by the PSQR under $\bTheta^*$ and by the Poisson approximation using the same $\bTheta^*$.

\begin{assumption}
\label{asm:2}
Let $\lambda_{ij}^*:= \exp\left(\theta_{jj}^*+\btheta_{j,-j}^{*\top}\bx_{i,-j}\right)$ be the mean parameter of a Poisson distribution. Then $\forall i \in \curly{1,2\cdots,n}$ and $\forall j \in \curly{1,2,\cdots,p}$, for some $C_3>0$ and  $C_4>0$, we have that $\mathbb{E} \left[X_j \given \bx_{i,-j} \right] \le C_3$ and 
$\abs{\lambda_{ij}^* -  \mathbb{E} \left[X_j \given \bx_{i,-j}\right]} \le C_4$.
\end{assumption}

The third assumption is the mutual incoherence condition vital to the sparsistency of sparse statistical learning with $l_1$-regularization. Also, with a slight abuse of notation, in the remaining of Section~\ref{sec:theory} as well as in the corresponding proofs, we should view $\bTheta$ as a vector generated by stacking up $\theta_{jj'}$'s, where $j \le j'$, whenever it is clear from context.

\begin{assumption}
\label{asm:3} 
Let $\bTheta^*$ be given. Define the index sets
\begin{gather*}
A := \curly{(j,j') \mid \theta_{jj'}^*\ne 0, j\ne j', j,j' \in \curly{1,2,\cdots,p} },\\
D := \curly{(j,j) \mid j \in \curly{1,2,\cdots,p}},\quad S := A \cup D,\\
I := \curly{(j,j') \mid \theta_{jj'}^*= 0, j\ne j', j,j' \in \curly{1,2,\cdots,p}}.
\end{gather*}
Let $\bH:=\grad^2 F(\bTheta^*)$. Then for some $0<\alpha<1$ and $C_5 > 0$, we have $\Norm{\bH_{IS}\bH_{SS}^{-1}}_{\infty} \le 1-\alpha$ and $\Norm{\bH_{SS}^{-1}}_{\infty}\le C_5$, where we use the index sets as subscripts to represent the corresponding components of a vector or a matrix.
\end{assumption}

The final assumption characterizes the second-order Taylor expansion of $F(\bTheta^*)$ at a certain direction $\bDelta$. 
\begin{assumption}
\label{asm:4}
Let $\bm{R}(\bDelta)$ be the second-order Taylor expansion remainder of $\grad F(\bTheta)$ around $\bTheta = \bTheta^*$ at direction $\bDelta:=\bTheta-\bTheta^*$ (i.e.~$\grad F(\bTheta) = \grad F(\bTheta^*) + \grad^2 F(\bTheta^*) (\bTheta-\bTheta^*)+ \bm{R}(\bDelta)$), where $\norm{\bDelta}_{\infty} \leq r:= 4C_5 \lambda \leq \frac{1}{C_5 C_6}$ with $\bDelta_I = \bzero$, and for some $C_6>0$. Then $\Norm{\bm{R}(\bDelta)}_{\infty} \le C_6 \Norm{\bDelta}_{\infty}^2$.
\end{assumption}

With these mild assumptions, the sparsistency result is stated in Theorem~\ref{thm:sparsistency}.

\begin{theorem}
\label{thm:sparsistency}
Suppose that Assumption~\ref{asm:1} - \ref{asm:4} are all satisfied.  Then, with probability of at least $1-\left(\left(\exp\left(C_1+{C_2}/{2}\right)+8\right)p^{-2}+p^{{-1}/{C_2}}\right)$, $\hat{\bTheta}$ shares the same structure with $\bTheta^*$, if for some constant $C_7>0$,
\begin{align*}
\lambda  &\geq  \frac{8}{\alpha} \left[C_3(3\log p + \log n) + (3\log p + \log n)^2 \right]\sqrt{\frac{\log p}{n}}
\\&+ 8C_4\left( C_1+ \sqrt{\frac{2\log p}{n}} \right)\alpha,\quad \lambda \leq C_7 \sqrt{\frac{\log^5 p}{n}},\\ 
r & \le \Norm{\bTheta^*_S}_{\infty}, \text{ and } n\ge \left(64C_7  C_5^2C_6/\alpha\right)^2\log^5 p.
\end{align*}
\end{theorem}

We defer the proof of Theorem~\ref{thm:sparsistency} to the Appendix. Note that $\log^5 p$ in Theorem~\ref{thm:sparsistency} represents a higher sample complexity compared to similar results in the analysis of Ising models \citep{ravikumar2010high}. Such a higher sample complexity intuitively makes sense since the multivariate count variables that we deal with are unbounded and usually heavy-tailed, and we are also considering the Poisson pseudo-likelihood approximation to the original pseudo-likelihood problem induced by PSQRs.
The fact that Poisson pseudo-likelihood is a sparsistent procedure for learning PSQRs not only provides an efficient approach to learn PSQRs with strong theoretical guarantees, but also establishes a formal correspondence between local Poisson graphical models  (LPGMs, \citealt{allen2013local}) and PSQRs. This is because Poisson pseudo-likelihood is also a sparsistent procedure for LPGMs. Compared to PSQRs, LPGMs are more intuitive yet less stringent theoretically due to the lack of a joint distribution defined by the model. Fortunately, with the guarantees in Theorem~\ref{thm:sparsistency}, we are able to provide some reassurance for the use of LPGMs in terms of structure recovery.


\begin{figure}[t]
\centering
\begin{subfigure}{0.69\linewidth}
\centering
\includegraphics[scale=0.46]{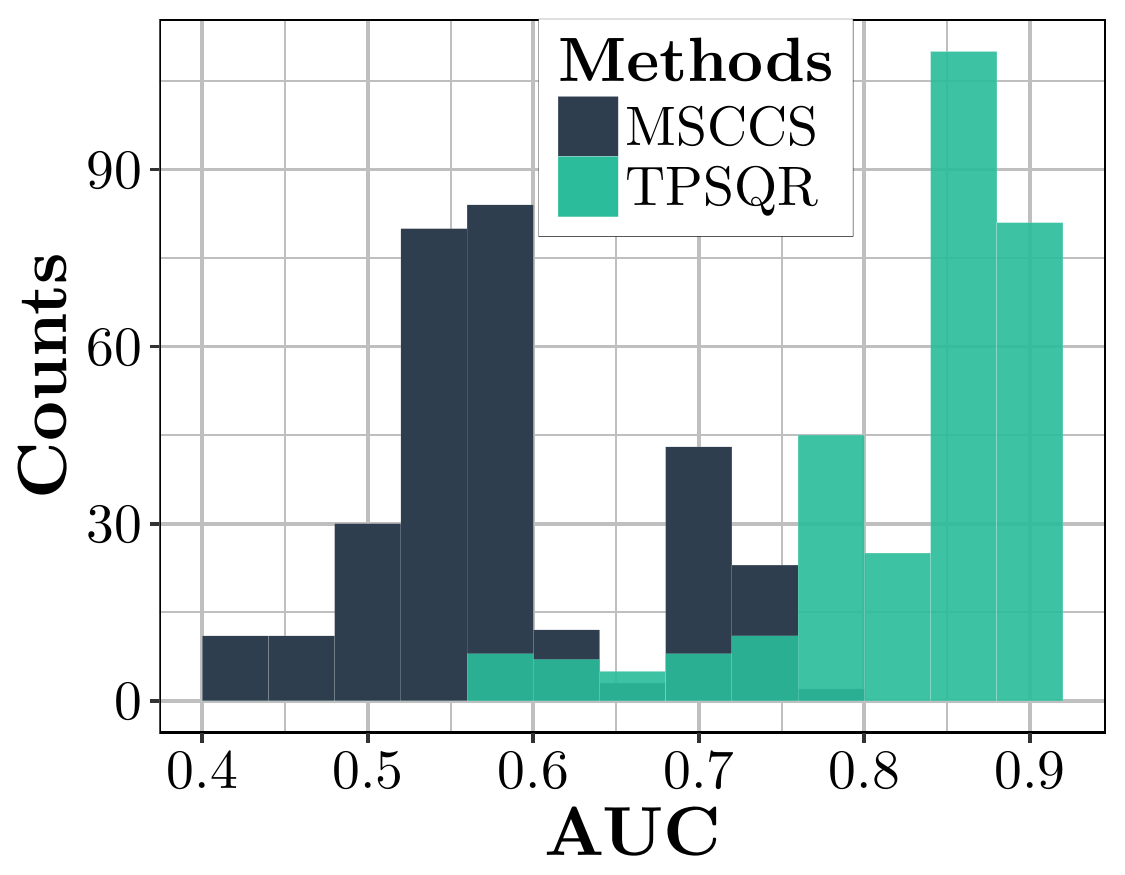}
\caption{Overlapping Histograms}
\end{subfigure}~
\begin{subfigure}{0.29\linewidth}
\centering
\includegraphics[scale=0.46]{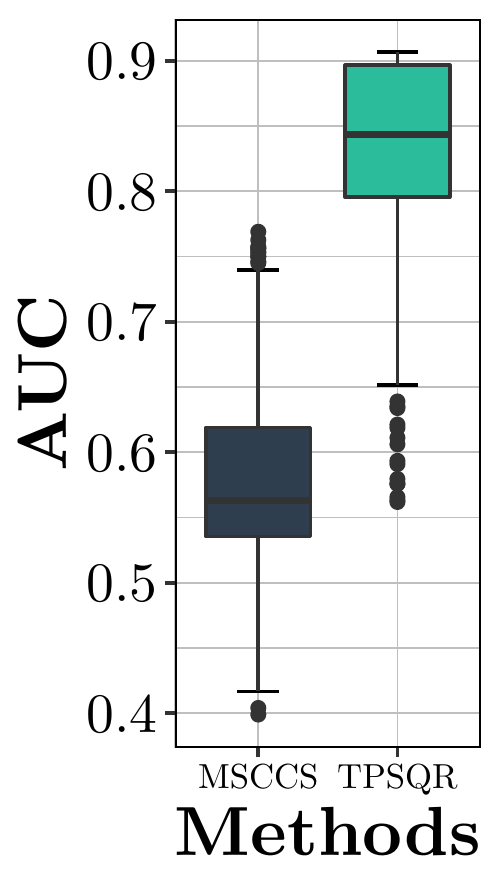}
\caption{Boxplot}
\end{subfigure}
\caption{Overall performance of TPSQR and MSCCS measured by AUC among 300 different experimental configurations for each of the two methods.}
\label{fig:overall}
\end{figure}

\section{Adverse Drug Reaction Discovery}
\label{sec:adr}
To demonstrate the capability of TPSQRs to capture temporal relationships between different pairs of event types in LED, we use ADR discovery from EHR as an example. ADR discovery is the task of finding unexpected and negative incidents caused by drug prescriptions. In EHR, time-stamped drug prescriptions as well as condition diagnoses are collected from millions of patients. These prescriptions of different drugs and diagnoses of different conditions can hence be viewed as various event types in LED. Therefore, using TPSQR, we can model whether the occurrences of a particular drug $k$ could elevate the possibility of the future occurrences of 
a condition $k'$ on different time scales by estimating $\bw_{kk'}$ defined in \eqref{eq:w}. If an elevation is observed, we can consider the drug $k$ as a potential candidate to cause condition $k'$ as an adverse drug reaction.

Postmarketing ADR surveillance from EHR is a multi-decade research and practice effort that is  of utmost importance to the pharmacovigilance community \citep{bate2018hope}, with substantial financial and clinical implication for health care delivery \citep{sultana2013clinical}.  Various ADR discovery methods have been proposed over the years \citep{harpaz2012novel}, and a benchmark task is created by the Observational Medical Outcome Partnership (OMOP, \citealt{simpson2011self}) to evaluate the ADR signal detection performance of these methods. The OMOP task is to identify the ADRs in 50 drug-condition pairs, coming from a selective combination of ten different drugs and nine different conditions. Among the 50 pairs, 9 of them are confirmed ADRs, while the remaining 41 of them are negative controls.

A most successful ADR discovery method using EHR is the multiple self-controlled case series (MSCCS, \citealt{simpson2013multiple}), which has been deployed in real-world ADR discovery related projects \citep{hripcsak2015observational}.  A reason for the success of MSCCS is  its introduction of fixed effects to address the heterogeneity among different subjects (e.g.~patients in poorer health might tend to be more likely to have a heart attack compared to a healthy person, which might confound the effects of various drugs when identifying drugs that could cause heart attacks as an ADR).

Therefore, when using TPSQR, we will also introduce fixed effects to equip TPSQRs with the capability of addressing subject heterogeneity. Specifically, we consider learning a variant of \eqref{eq:pseudo-poisson}:
\begin{align*}
\max_{\balpha,\bomega,\bW} & \frac{1}{M} \hspace{-1mm}  \sum_{i=1}^N \sum_{j=1}^{n_i} \hspace{-1mm} \log \hat{\text{P}} \left( \hspace{-0.5mm} x_{ij} \given \bx_{i,-j}; \btheta^{(i)}_j\hspace{-0.5mm}, \alpha _{i o_{ij}}\right)\hspace{-1mm}- \hspace{-1mm}\lambda \norm{\bW}_{1,1},
\end{align*}
where $\balpha$ is the fixed effect parameter vector constructed by $\alpha_{io_{ij}}$'s that depicts the belief that different patients could have different baseline risks of experiencing different types of events with $\hat{\text{P}} \left( \hspace{-0.5mm} x_{ij} \given \bx_{i,-j}; \btheta^{(i)}_j\hspace{-0.5mm}, \alpha _{i o_{ij}}\right)\propto {\exp\left[ \left(\hspace{-1mm} \alpha _{i o_{ij}}  + {\btheta_{-j}^{(i)\top}} \bx_{i,-j} \right)  x_{ij} \right]}/{ x_{ij}!}$.

\section{Experiments}
\label{sec:experiment}
In what follows, we will compare the performances 
of TPSQR, MSCCS, and Hawkes process \citep{bao2017hawkes} in the OMOP task. The experiments are conducted using Marshfield Clinic EHRs with millions of drug prescription and condition diagnosis events from 200,000 patients.

\subsection{Experimental Configuration}
\textbf{Minimum Duration}: clinical encounter sequences from different patients might span across different time lengths. Some have decades of observations in their records while other might have records only last a few days. We therefore consider minimum duration of the clinic encounter sequence as a threshold to determine whether we admit a patient to the study or not. In our experiments, we consider two minimum duration thresholds: \emph{0.5 year} and \emph{1 year}.


\textbf{Maximum Time Difference}: for TPSQR,  in \eqref{eq:phi}, $\tau_L$ determines the maximum time difference between the occurrences of two events within which the former event might have nonzero influence on the latter event. We call $\tau_L$ the maximum time difference to characterize how distant in the past we would like to take previous occurrences into consideration when modeling future occurrences. In our experiments, we consider three maximum time differences: \emph{0.5 year}, \emph{1 year}, and \emph{1.5 years}. $L=3$ and the corresponding influence functions are chosen according to \citealt{bao2017hawkes}. In MSCCS, a configuration named \emph{risk window} serves a similar purpose to the maximum difference in TPSQR. We choose three risk windows according to \citealt{kuang2017pharmacovigilance} so as to ensure that the both TPSQR and MSCCS have similar capability in considering the event history on various time scales.

\textbf{Regularization Parameter}: we use $l_1$-regularization for TPSQR since it encourages sparsity, and the sparsity patterns learned correspond to the structures of the graphical models. We use $L_2$-regularization for MSCCS since it yields outstanding empirical performance in previous studies \citep{simpson2013multiple, kuang2017pharmacovigilance}. 50 regularization parameters are chosen for both TPSQR and MSCCS. 

To sum up, \comment{across the two minimum duration thresholds, the three maximum time differences (risk windows), and the 50 regularization parameters for each of the two models, }there are $2\times 3 \times 50 = 300$ experimental configurations respectively for TPSQR and MSCCS.

\begin{figure}[t]
\centering
\includegraphics[scale=0.4]{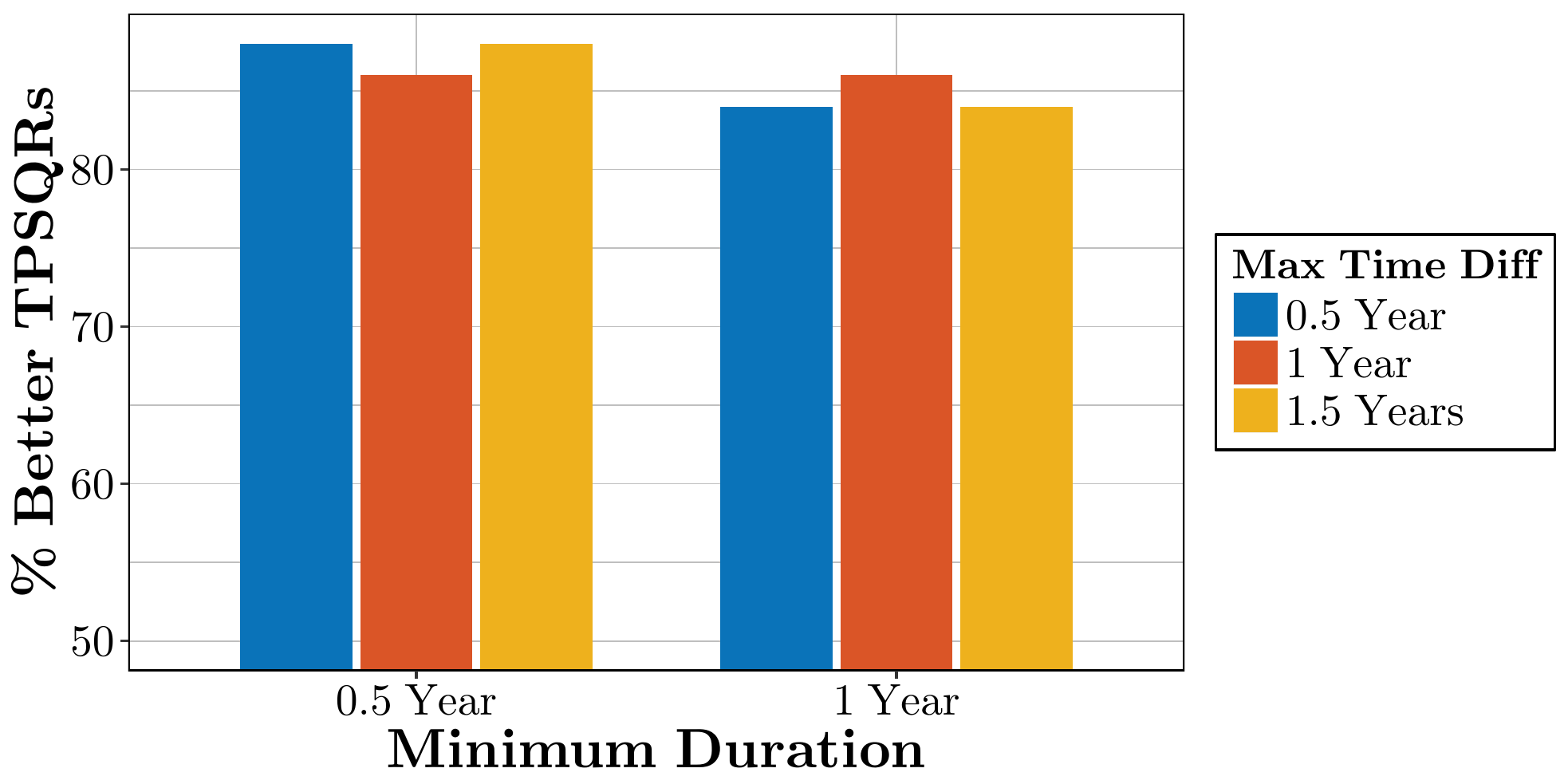}
\caption{Percentage of better TPSQR models under various minimum duration and maximum time difference designs compared to the best MSCCS model}
\label{fig:sensitivity}
\end{figure}

\begin{figure}[t]
\centering
\includegraphics[scale=0.4]{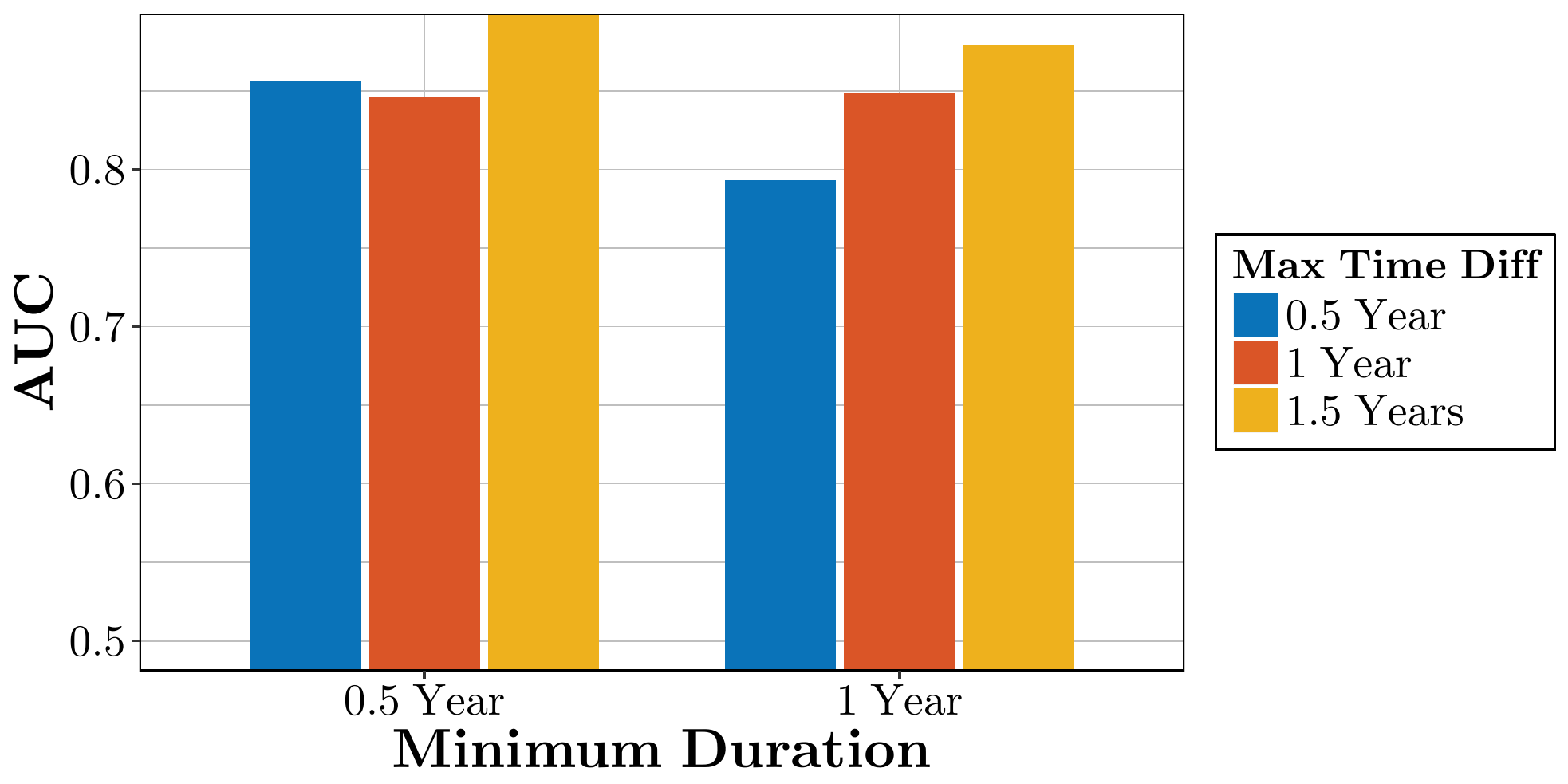}
\caption{AUC of TPSQR models selected by AIC for given minimum duration and maximum time difference designs}
\label{fig:aic}
\end{figure}

\vspace{-2mm}
\subsection{Overall Performance}
For each of the 300 experimental configurations for TPSQR and MSCCS, we perform the OMOP task using our EHR data. Both TPSQR and MSCCS can be implemented by the \texttt{R} package \texttt{glmnet} \citep{friedman2010regularization, tibshirani2012strong}. We then use Area Under the Curve (AUC) for the receiver operating characteristic curve to evaluate how well TPSQR and MSCCS can distinguish actual ADRs from negative controls under this particular experimental configuration. The result is 300 AUCs corresponding to the total number of experimental configurations for each of the two methods. For TPSQR, since the effect of drug $k$ on condition $k'$ is estimated over different time scales via $\bw_{kk'}$, the score corresponding to this drug-condition pair used to calculate the AUC is computed by the average over all the components of $\bw_{kk'}$. For MSCCS, AUC is computed according to \citealt{kuang2017pharmacovigilance}.
Figure~\ref{fig:overall} presents the histogram of these two sets of 300 AUCs. The contrast in the performances between TPSQR and MSCCS is obvious. The distribution of TPSQR shifts substantially towards higher AUC values compared to the distribution of MSCCS. Therefore, the overall performance of TPSQR is superior to that of MSCCS in the OMOP task under various experimental configurations in question. As a matter of fact, the top performing TPSQR model reaches an AUC of 0.91, as opposed to 0.77 for MSCCS. Furthermore, the majority of TPSQRs have higher AUCs even compared to the MSCCS model that has the best AUC. We also contrast the performance of TPSQR with the Hawkes process method in \citealt{bao2017hawkes}, whose best AUC is $0.84$ under the same experiment configurations.

\vspace{-2mm}
\subsection{Sensitivity Analysis and Model Selection}
To see how sensitive the performance of TPSQR is for different choices of experimental configurations, we compute the percentage of TPSQRs with a given minimum duration and a given maximum time difference design that are better than the best MSCCS model (with an AUC of 0.77). The results are summarized in Figure~\ref{fig:sensitivity}. As can be seen, the percentage of better TPSQRs is consistently above 80\% under various scenarios, suggesting the robustness of TPSQRs to various experimental configurations. Given a fixed minimum duration and a fixed maximum time difference, we conduct model selection for TPQSRs by the Akaike information criterion (AIC) over the regularization parameters. The AUC of the selected models are summarized in Figure~\ref{fig:aic}. Note that under various fixed minimum duration and maximum time difference designs, AIC is capable of selecting models with high AUCs. In fact, all the models selected by AIC have higher AUCs than the best performer of MSCCS. This phenomenon demonstrates that the performance of TPSQR is consistent and robust with respect to the various choices of experimental configurations.

\vspace{-2mm}
\section{Conclusion}
\label{sec:conclusion}

We propose TPSQRs, a generalization of PSQRs for the temporal relationships between different event types in LED. We propose the use of Poisson pseudo-likelihood approximation to solve the pseudo-likelihood problem arising from PSQRs. The approximation procedure is extremely efficient to solve, and is sparsistent in recovering the structure of the underlying PSQR. The utility of TPSQR is demonstrated using Marshfield Clinic EHRs for adverse drug reaction discovery.

{\small \textbf{Acknowledgement}:  The authors would like to gratefully acknowledge the NIH BD2K Initiative grant U54 AI117924 and the NIGMS grant 2RO1 GM097618.}

\bibliographystyle{plainnat}
\bibliography{ref}

\newpage\clearpage\onecolumn
\input{supp.tex}

\end{document}

%% file: supp.tex
\section{Appendix}
\label{sec:appendix}

We prove Theorem~\ref{thm:sparsistency} in this section. Since the proof is technical and lengthy, for the clarity of presentation, we organize the proof as follows. To begin with, in Section~\ref{sec:concentrate},  we review two standard concentration inequalities, the Chernoff inequality and the Hoeffding inequality, which will be used to prove some technical lemmas. We then present and prove these technical lemmas in Section~\ref{sec:lemma}. These technical lemmas are subsequently used to validate some auxiliary results, which are presented in Section~\ref{sec:auxiliary}. Finally, we prove Theorem~\ref{thm:sparsistency} based on these auxiliary results.

\subsection{Concentration Inequalities}
\label{sec:concentrate}
\begin{lemma}[Hoeffding Inequality]
\label{lem:hoeffding}
Let $X_1,X_2,\cdots,X_n$ be $n$ i.i.d.~random variables drawn from the distribution $\mathcal{D}$, with  $0 \leq X_i \leq a$, $\forall i \in \curly{1,2,\cdots,n}$. Let $\bar{X} := \frac{1}{n}\sum_{i=1}^n X_i$. Then, for any $t >0$, 
\begin{equation*}
\text{P}(\abs{\bar{X} - \mathbb{E}[\bar{X}] }\geq t ) \leq 2\exp\left(-\frac{2nt^2}{a^2} \right).
\end{equation*}
\end{lemma}

\begin{lemma}[Chernoff Inequality]
\label{lem:chernoff}
Let $X_1,X_2,\cdots,X_n$ be $n$ random variables and let $X := \sum_{i=1}^n X_i$.
Then, for any $t >0$, 
\begin{equation}
\label{eq:chernoff-1}
\text{P}(X\geq \epsilon ) \leq \exp(-t\epsilon) \mathbb{E}\left[  \exp\left( \sum_{i=1}^n t X_i \right) \right].
\end{equation}
Furthermore, if $X_i$'s are independent, then
\begin{equation}
\label{eq:chernoff-2}
\text{P}(X\geq \epsilon ) \leq \min_{t>0} \exp(-t\epsilon) \prod_{i=1}^n  \mathbb{E}\left[ \exp(t X_i) \right].
\end{equation}
\end{lemma}

\subsection{Technical Lemmas}
\label{sec:lemma}

We use $\norm{\cdot}_\max$ to represent the max norm of a matrix, which is equal to the maximum of the absolute value of all the elements in the matrix.

\begin{lemma}
\label{lem:con-jk}
Let $\mathbb{X}$ be given. Suppose that $0<\max_{i,i' \in \curly{1,2,\cdots,n}} \norm{\bx_i \bx_{i'}^\top}_\max<\epsilon^2$. Then,
\begin{align*}
\emph{P}\left(\max_{j\ne j', j\ne j', j,j'\in\curly{1,2,\cdots,p}}\abs{\mathbb{E}_{\mathbb{X}}[X_jX_{j'}] - \mathbb{E}[X_jX_{j'}]} \geq \epsilon^2\sqrt{\frac{\log p}{n}}\right) \leq 2\exp&(-2\log p).
\end{align*}
\end{lemma}
\begin{proof}
Since $0<\max_{i,i' \in \curly{1,2,\cdots,n}} \norm{\bx_i \bx_{i'}^\top}_\max<\epsilon^2$, we let $a=\epsilon^2$ and $t=\epsilon^2 \sqrt{\frac{\log p}{n}}$ in Lemma~\ref{lem:hoeffding} to yield the result.
\end{proof}

\begin{lemma}
\label{lem:con-j}
Let $\mathbb{X}$ be given.  Suppose that $0<\max_{i \in \curly{1,2,\cdots,n}} \norm{\bx_i}_\infty <\epsilon$. Then, 
\begin{align*}
\emph{P}\left(\max_{j\in\curly{1,2,\cdots,p}}\abs{\mathbb{E}_{\mathbb{X}}[X_j] - \mathbb{E}[X_j] } \geq \epsilon\sqrt{\frac{\log p}{n}}\right) \leq 2\exp&(-2\log p).
\end{align*}
\end{lemma}
\begin{proof}
Since $0<\max_{i \in \curly{1,2,\cdots,n}} \norm{\bx_i}_\infty <\epsilon$, we let $a=\epsilon$ and $t=\epsilon \sqrt{\frac{\log p}{n}}$ in Lemma~\ref{lem:hoeffding} to yield the result.
\end{proof}

\begin{lemma}
\label{lem:con-y}
Let $\mathbb{X}$ be given. Suppose that  $0<\max_{i \in \curly{1,2,\cdots,n}} \norm{\bx_i}_\infty <\epsilon$. Then,
\begin{equation*}
\emph{P} \left(\max_{j,j'\in\curly{1,2,\cdots,p}}\abs{\mathbb{E}_{\mathbb{X}}[ \mathbb{E}[X_j X_{j'} \given \mathbf{X}_{-j}] ] -\mathbb{E}[\mathbb{E}[X_j X_{j'} \given \bX_{-j}]]} \geq C_3\epsilon\sqrt{\frac{\log p}{n}}\right) \leq 2\exp(-2\log p).
\end{equation*}
\end{lemma}
\begin{proof}
Since $0<\max_{i \in \curly{1,2,\cdots,n}} \norm{\bx_i}_\infty <\epsilon$ and $\mathbb{E}[X_j \given \bx_{i,-j}] \le C_3$ by Assumption~\ref{asm:2}, we have that $ 0 <\mathbb{E}[X_j X_{j'} \given \bx_{i,-j}] \le C_3 \epsilon$. Therefore, we let $a= C_3 \epsilon$ and $t=C_3\epsilon\sqrt{\frac{\log p}{n}}$ in Lemma~\ref{lem:hoeffding} to yield the result.
\end{proof}

\subsection*{Remark}
The subtlety of the definitions of $C_3$ and $C_4$ in Assumption~\ref{asm:2}, as well as the notion of $\epsilon$ in Lemma~\ref{lem:con-jk}, Lemma~\ref{lem:con-j}, and Lemma~\ref{lem:con-y} should be noted. Formally, the $n$ data points $\bx_1$, $\bx_2$, $\cdots$, $\bx_n$ in $\mathbb{X}$ can be viewed as  assignments to the corresponding random variables $\bX^{(1)}$, $\bX^{(2)}$, $\cdots$, $\bX^{(n)}$ following the PSQR parameterized by $\bTheta^*$. In Assumption~\ref{asm:2}, we are interested in a set $\mathcal{X} \subseteq \mathbb{N}^p$, such that $\forall i \in  \curly{1,2,\cdots,n}$ and $\forall j \in \curly{1,2,\cdots,p}$, 
\begin{equation*}
\max_{\bX^{(i)} \in \mathcal{X}} \mathbb{E} \left[X_j \given \bX_{-j}^{(i)} \right] \le C_3 \quad\text{and}\quad \max_{\bX^{(i)} \in \mathcal{X}} \abs{\lambda_{ij}^* -  \mathbb{E} \left[X_j \given \bX_{-j}^{(i)} \right]} \le C_4.
\end{equation*}
In Lemma~\ref{lem:con-jk}, Lemma~\ref{lem:con-j}, and Lemma~\ref{lem:con-y}, 
we are interested in a set $\mathcal{X} \subseteq \mathbb{N}^p$, such that $\forall i, i' \in  \curly{1,2,\cdots,n}$, where $i \ne i'$,
\begin{equation*}
0<\max_{\bX^{(i)},\bX^{(i')} \in \mathcal{X}} \norm{\bX^{(i)} \bX^{(i')\top}}_\max<\epsilon^2 \quad \text{and} \quad 0<\max_{\bX^{(i)} \in \mathcal{X}} \norm{\bX^{(i)}}_\infty<\epsilon.
\end{equation*}
Also, implicitly, we have that $\bx_i \in \mathcal{X}$, $\forall i \in \curly{1,2,\cdots,n}$. 

\begin{lemma}
\label{lem:con-yj}
Let $\mathbb{X}$ be given. Then,
\begin{equation*}
\emph{P} \left(\max_{j\in\curly{1,2,\cdots,p}}\abs{\mathbb{E}_{\mathbb{X}}[\mathbb{E}[X_j \given \bX_{-j}]]-\mathbb{E}[\mathbb{E} [X_j \given \bX_{-j}]]} \geq C_3\sqrt{\frac{\log p}{n}}\right) \leq 2\exp(-2\log p).
\end{equation*}
\end{lemma}
\begin{proof}
Since $\mathbb{E}[X_j \given \bx_{i,-j}] \le C_3$ by Assumption~\ref{asm:2}, we have that $ 0 <\mathbb{E}[X_j \given \bx_{i,-j}] \le C_3$. Therefore, we let $a= C_3$ and $t=C_3\sqrt{\frac{\log p}{n}}$ in Lemma~\ref{lem:hoeffding} to yield the result.
\end{proof}

\begin{lemma}
\label{lem:bound-con}
Let $\bX$ be a random vector drawn from a PSQR distribution parameterized by $\bTheta^*$. Suppose that $\left\{\bx_1, \bx_2, \cdots \bx_n\right\}^{\top}$ is the set of $n$ i.i.d.~samples of $\bX$. Given $j  \in \{1,2,\cdots, p\}$, $\epsilon_1 := 3\log p + \log n$, and $\epsilon_2 := C_1 + \sqrt{\frac{2\log p}{n}}$,
\begin{equation*}
\emph{P} \left( X_j \geq \epsilon_1 \right) \leq \exp( C_1 + C_2/2 - \epsilon_1), \text{\quad and\quad} \emph{P} \left(  \frac{1}{n} \sum_{i=1}^n x_{ij} \geq \epsilon_2  \right) \leq \exp \left[-\frac{n (\epsilon_2-C_1)^2}{2 C_2} \right].
\end{equation*}
\end{lemma}

\begin{proof}
We start with proving the first inequality. To this end, consider the following equation due to Taylor expansion:
\begin{align}
\label{eq:moment}
\begin{split}
\log \mathbb{E} \left[ \exp(X_j)\right] 
= &B(\bTheta^*, \bm{0}+ \be_j ) - B(\bTheta^*, \bm{0}) 
= \grad^\top B(\bTheta^*, \bm{0})  \be_j + \frac{1}{2} \be_j^\top \grad^2 B(\bTheta^*, k \be_j)  \be_j\\
= &\mathbb{E} [X_j] + \frac{1}{2} \frac{\partial^2}{\partial b_j^2} B(\bTheta^*,\bm{0} + k\be_j) \leq C_1 + C_2/2,
\end{split}
\end{align}
where $k \in [0,1]$, $\be_j$ is a vector whose $j^{\text{th}}$ component is one and zeros elsewhere, and the last inequality is due to Assumption~\ref{asm:1}. Then, let $t=1$ and $X=X_j$ in Lemma~\ref{lem:chernoff}, 
\begin{equation*}
\text{P} \left( X_j \geq \epsilon_1 \right) =  \exp(-\epsilon_1) \mathbb{E}\left[ \exp(X_j) \right] \le \exp(C_1+C_2/2 - \epsilon_1).
\end{equation*}

Now, we prove the second bound. For any $0<a<1$ and some $k \in [0,1]$, with Taylor expansion,
\begin{align}
\label{eq:moment-a}
\begin{split}
\log \mathbb{E} \left[ \exp(aX_i)\right] =& B(\bTheta^*, \bm{0}+ a\be_j ) - B(\bTheta^*, \bm{0}) 
= a\grad^\top B(\bTheta^*, \bm{0})  \be_j + \frac{a^2}{2}\be_j^\top \grad^2 B(\bTheta^*, \bm{0} + ak\bm{e}_j)  \be_j \\
=& a\mathbb{E}(X_j) + \frac{a^2}{2} \frac{\partial^2}{\partial b_j^2} B(\bTheta^*,\bm{0} + ak\be_j) \le a C_1 + \frac{a^2}{2}C_2,
\end{split}
\end{align}
where the last inequality is due to Assumption~\ref{asm:1}. Then, following the proof technique above, we have
\begin{align*}
\text{P}\left(\frac{1}{n} \sum_{i=1}^n X_i \geq \epsilon_2 \right) =& \text{P}\left( \sum_{i=1}^n X_i \geq n\epsilon_2 \right) \leq \min_{t>0} \exp(-t n \epsilon_2) \prod_{i=1}^n  \mathbb{E}\left[ \exp\left(t X_i\right) \right] \\
\leq & \min_{t>0} \exp(-tn\epsilon_2) \prod_{i=1}^n \exp \left(C_1 t + \frac{C_2}{2}t^2\right)\\
 = & \min_{t>0} \exp  \left[(C_1-\epsilon_2)nt + \frac{n C_2}{2}t^2\right]  \leq \exp \left[-\frac{n (\epsilon_2-C_1)^2}{2 C_2} \right],
\end{align*}
where the minimum is obtained when $t=\frac{\epsilon_2-C_1}{C_2}$, and we have used the fact that $\epsilon_2 > C_1$.
\end{proof}

\subsection{Auxiliary Results}
\label{sec:auxiliary}
\begin{lemma}
\label{lem:bound-der}
Let $r :=  4C_5\lambda$. Then with probability of at least $1-\left(\left(\exp\left(C_1+{C_2}/{2}\right)+8\right)p^{-2}+p^{{-1}/{C_2}}\right)$, the following two inequalities simultaneously hold:
\begin{equation}
\label{eq:bound-der-1}
\norm{\grad F(\bTheta^*)}_\infty \leq 2 \left[C_3(3\log p + \log n) + (3\log p + \log n)^2 \right]\sqrt{\frac{\log p}{n}} + 2C_4\left( C_1+ \sqrt{\frac{2\log p}{n}} \right),
\end{equation}
\begin{equation}
\label{eq:bound-der-2}
\norm{\tilde{\bTheta}_S - \bTheta^*_S}_\infty \leq r.
\end{equation}
\end{lemma}

\begin{proof} 
We prove \eqref{eq:bound-der-1} and \eqref{eq:bound-der-2} in turn. 

\subsubsection*{Proof of \eqref{eq:bound-der-1}}
To begin with, we prove \eqref{eq:bound-der-1}. By the definition of $F$ in \eqref{eq:F}, for $j<j'$, the derivative of $F(\bTheta^*)$ is:
\begin{equation}
\label{eq:der-jk}
\frac{\partial F(\bTheta^*)}{\partial \theta_{jj'}} = \frac{1}{n} \sum_{i=1}^n \left[-x_{ij'}x_{ij} + \lambda_{ij}^* x_{ij'} - x_{ij}x_{ij'} + \lambda_{ij'}^* x_{ij}\right] = -2\mathbb{E}_{\mathbb{X}} [X_j X_{j'}] + \frac{1}{n}\sum_{i=1}^n \lambda_{ij}^* x_{ij'} + \frac{1}{n}\sum_{i=1}^n \lambda_{ij'}^* x_{ij}.
\end{equation}
and 
\begin{equation}
\label{eq:der-jk-diag}
\frac{\partial}{\partial \theta_{jj}}F(\bTheta^*) = \frac{1}{n}\sum_{i=1}^n \left[-x_{ij} + \lambda_{ij}^* \right] = -\mathbb{E}_{\mathbb{X}} [X_j ] + \frac{1}{n}\sum_{i=1}^n \lambda^*_{ij},
\end{equation}
 where $\mathbb{E}_{\mathbb{X}} [X_j X_{j'}] := \frac{1}{n}\sum_{i=1}^n x_{ij}x_{ij'}$ and $\mathbb{E}_{\mathbb{X}} [X_j ] := \frac{1}{n}\sum_{i=1}^n x_{ij}$ are the expectations of $X_j X_{j'}$ and $X_j$ over the empirical distribution given by the dataset $\mathbb{X}$.

Then, by defining $\mathbb{E}[X_j X_{j'}]$ as the expectation of the multiplication of two components of an multivariate square root Poisson random vector whose distribution is parameterized by $\bTheta^*$, and by Assumption~\ref{asm:2}, \eqref{eq:der-jk} can be controlled via
\allowdisplaybreaks
\begin{align*}
\abs{\frac{\partial}{\partial \theta_{jj'}}F(\bTheta^*)} 
= & \abs{\frac{1}{n}\sum_{i=1}^n \lambda_{ij}^* x_{ij'} -\mathbb{E}[X_j X_{j'}] + \frac{1}{n}\sum_{i=1}^n \lambda_{ij'}^* x_{ij}-\mathbb{E}[X_j X_{j'}] +2\mathbb{E}[X_j X_{j'}]-2\mathbb{E}_{\mathbb{X}} [X_j X_{j'}] }\\
\leq & \abs{\frac{1}{n}\sum_{i=1}^n \lambda_{ij}^* x_{ij'} -\mathbb{E}[X_j X_{j'}]} + \abs{\frac{1}{n}\sum_{i=1}^n \lambda_{ij'}^* x_{ij}-\mathbb{E}[X_j X_{j'}]} + 2\abs{\mathbb{E}_{\mathbb{X}} [X_j X_{j'}]  - \mathbb{E}[X_j X_{j'}] }\\
= & \abs{\frac{1}{n}\sum_{i=1}^n \left(\mathbb{E}[X_j \given \bX_{-j} = \bx_{i,-j}] + \lambda_{ij}^*-\mathbb{E}[X_j \given \bX_{-j} = \bx_{i,-j}]\right) x_{ij'} -\mathbb{E}[X_j X_{j'}]} \\
+ & \abs{\frac{1}{n}\sum_{i=1}^n \left( \mathbb{E}[X_{j'} \given \bX_{-j'} = \bx_{i,-j'}] + \lambda_{ij'}^*-\mathbb{E}[X_{j'} \given \bX_{-j'} = \bx_{i,-j'}] \right) x_{ij}-\mathbb{E}[X_j X_{j'}]} \\
+ & 2\abs{\mathbb{E}_{\mathbb{X}} [X_j X_{j'}]  - \mathbb{E}[X_j X_{j'}] }\\
\leq & \abs{\frac{1}{n}\sum_{i=1}^n \left(\mathbb{E}[X_j \given \bX_{-j} = \bx_{i,-j}] \right) x_{ij'} -\mathbb{E}[X_j X_{j'}]} + \frac{1}{n}\sum_{i=1}^n \abs{\lambda_{ij}^*-\mathbb{E}[X_j \given \bX_{-j} = \bx_{i,-j}]} x_{ij'}\\
+ & \abs{\frac{1}{n}\sum_{i=1}^n \left(\mathbb{E}[X_{j'} \given \bX_{-j'} = \bx_{i,-j'}] \right) x_{ij} -\mathbb{E}[X_j X_{j'}]} + \frac{1}{n}\sum_{i=1}^n \abs{\lambda_{ij'}^*-\mathbb{E}[X_{j'} \given \bX_{-j'} = \bx_{i,-j'}]} x_{ij}\\
+ & 2\abs{\mathbb{E}_{\mathbb{X}} [X_j X_{j'}]  - \mathbb{E}[X_j X_{j'}] }\\
\leq & \abs{ \frac{1}{n}\sum_{i=1}^n \mathbb{E}[X_j \given \bX_{-j} = \bx_{i,-j}] x_{ij'} -\mathbb{E}[X_j X_{j'}]}
+ \abs{\frac{1}{n}\sum_{i=1}^n \mathbb{E}[X_{j'} \given \bX_{-j'} = \bx_{i,-j'}] x_{ij}-\mathbb{E}[X_j X_{j'}]} \\
+ & 2\abs{\mathbb{E}_{\mathbb{X}} [X_j X_{j'}]  - \mathbb{E}[X_j X_{j'}] }+ C_4(\mathbb{E}_{\mathbb{X}} [X_j]  + \mathbb{E}_{\mathbb{X}} [X_{j'}] )\\
= & \abs{ \frac{1}{n}\sum_{i=1}^n \mathbb{E}[X_j X_{j'} \given \bX_{-j} = \bx_{i,-j}] -\mathbb{E}[X_j X_{j'}]}
+ \abs{\frac{1}{n}\sum_{i=1}^n \mathbb{E}[X_j X_{j'} \given \bX_{-j'} = \bx_{i,-j'}]-\mathbb{E}[X_j X_{j'}]} \\
+ & 2\abs{\mathbb{E}_{\mathbb{X}} [X_j X_{j'}] - \mathbb{E}[X_j X_{j'}]  }+ C_4(\mathbb{E}_{\mathbb{X}} [X_j]  + \mathbb{E}_{\mathbb{X}} [X_{j'}] )\\
= & 2 \abs{ \mathbb{E}_{\mathbb{X}}[ \mathbb{E}[X_j X_{j'} \given \mathbf{X}_{-j}] ] -\mathbb{E}[X_j X_{j'}] } + 2\abs{\mathbb{E}_{\mathbb{X}}[ X_j X_{j'}] - \mathbb{E}[X_j X_{j'}] }+ C_4(\mathbb{E}_{\mathbb{X}}[ X_j] + \mathbb{E}_{\mathbb{X}}[ X_{j'} ])\\
= & 2 \abs{ \mathbb{E}_{\mathbb{X}}[ \mathbb{E}[X_j X_{j'} \given \mathbf{X}_{-j}] ] -\mathbb{E}[\mathbb{E}[X_j X_{j'} \given \bX_{-j}]] } + 2\abs{\mathbb{E}_{\mathbb{X}}[ X_j X_{j'}] - \mathbb{E}[X_j X_{j'}] }+ C_4(\mathbb{E}_{\mathbb{X}}[ X_j] + \mathbb{E}_{\mathbb{X}}[ X_{j'} ]),
\end{align*} 
where we have used the law of total expectation in the last equality.

Similarly, \eqref{eq:der-jk-diag} can be controlled via
\begin{align*}
\abs{\frac{\partial}{\partial \theta_{jj}}F(\bTheta^*)} = &
\abs{-\mathbb{E}_{\mathbb{X}} [X_j ] + \frac{1}{n}\sum_{i=1}^n \lambda^*_{ij}} 
= \abs{-\mathbb{E}_{\mathbb{X}} [X_j ] + \frac{1}{n}\sum_{i=1}^n \left(\mathbb{E}[X_j \given \bX_{-j} = \bx_{i,-j}] + \lambda^*_{ij} - \mathbb{E}[X_j \given \bX_{-j} = \bx_{i,-j}]\right)}  \\
= & \abs{-\mathbb{E}_{\mathbb{X}} [X_j ] + \mathbb{E}[X_j] - \mathbb{E}[X_j] + \mathbb{E}_{\mathbb{X}}[\mathbb{E}[X_j \given \bX_{-j}]] + \frac{1}{n}\sum_{i=1}^n \left(\lambda^*_{ij} - \mathbb{E}[X_j \given \bX_{-j} = \bx_{i,-j}]\right)}  \\
\leq & \abs{\mathbb{E}_{\mathbb{X}}[\mathbb{E}[X_j \given \bX_{-j}]]-\mathbb{E}[X_j ] } + \abs{\mathbb{E}_{\mathbb{X}}[X_j]-\mathbb{E}[X_j]}  + C_4 \\
=  & \abs{\mathbb{E}_{\mathbb{X}}[\mathbb{E}[X_j \given \bX_{-j}]]-\mathbb{E}[\mathbb{E} [X_j \given \bX_{-j}] ] } + \abs{\mathbb{E}_{\mathbb{X}}[X_j]-\mathbb{E}[X_j]}  + C_4.
\end{align*} 

We define four events: 
\begin{gather*}
E_1 := \curly{\max_{j\ne j', j,j'\in\curly{1,2,\cdots,p}}\abs{\frac{\partial}{\partial \theta_{jj'}}F(\bTheta^*)} \geq 2(C_3\epsilon_1+\epsilon^2_1)  \sqrt{\frac{\log p}{n}} +2 C_4 \epsilon_2 },\\ 
E_2 := \curly{\max_{j\in\curly{1,2,\cdots,p}} \abs{\frac{\partial}{\partial \theta_{jj}}F(\bTheta^*)} \geq   (C_3+\epsilon_1)\sqrt{\frac{\log p}{n}} + C_4/n },\\
E_3 := \curly{ 0<\max_{i\in\curly{1,2,\cdots,n}}\norm{\bx_i}_{\infty}<\epsilon_1},\quad \text{and}\quad  E_4 := \curly{ 0<\max_{j \in \curly{1,2,\cdots,p}}\mathbb{E}_{\mathbb{X}}[X_j]<\epsilon_2},
\end{gather*}
where $\epsilon_1 := 3\log p + \log n$ and $\epsilon_2 := C_1 + \sqrt{\frac{2\log p}{n}}$ are defined in Lemma~\ref{lem:bound-con}. By Lemma \ref{lem:con-jk}, Lemma~\ref{lem:con-j}, Lemma~\ref{lem:con-y} and Lemma \ref{lem:con-yj}, it follows that 
\begin{equation}
\label{eq:e1-or-e2-1}
\text{P}(E_1\given E_3, E_4) \leq 4\exp(-2\log p)  \text{\quad and \quad} \text{P}(E_2\given E_3, E_4) \leq 4\exp(-2\log p).
\end{equation}
Therefore,
\begin{align}
\label{eq:e1-or-e2-2}
\begin{split}
\text{P}(E_1 \cup E_2) = & \text{P}(E_1 \cup E_2 \mid E_3,E_4) \text{P}(E_3,E_4)  +  \text{P}(E_1 \cup E_2 \mid E_3^c,E_4) \text{P}(E_3^c,E_4)  \\
+ & \text{P}(E_1 \cup E_2 \mid  E_3,E_4^c) \text{P}(E_3,E_4^c) + 
\text{P}(E_1 \cup E_2 \mid E_3^c, E_4^c) \text{P}(E_3^c,E_4^c)  \\ 
\leq & \text{P}(E_1\given E_3, E_4) +\text{P}(E_2\given E_3, E_4) + \text{P}(E_3^c,E_4) + \text{P}(E_3,E_4^c) + \text{P}(E_3^C,E_4^c)  \\
\leq &\text{P}(E_1\given E_3, E_4) +\text{P}(E_2\given E_3, E_4) + \text{P}(E_3^c) + \text{P}(E_4^c) \\
\leq & 8\exp(-2\log p) + \exp( C_1 + C_2/2 -\epsilon_1 )np + \exp \left[-\frac{n (\epsilon_2-C_1)^2}{2 C_2} \right]\\
= & 8\exp(-2\log p) + \frac{\exp(C_1+{C_2}/{2})}{p^2} + p^{-\frac{1}{C_2}},
\end{split}
\end{align}
where the superscript $c$ over an event represents the complement of that event, and the last inequality is due to \eqref{eq:e1-or-e2-1} and Lemma~\ref{lem:bound-con}. Also notice that by the definitions of $E_1$ and $E_2$,
\begin{equation*}
2(C_3\epsilon_1+\epsilon^2_1)  \sqrt{\frac{\log p}{n}} +2 C_4 \epsilon_2 > (C_3+\epsilon_1)\sqrt{\frac{\log p}{n}} + C_4/n .
\end{equation*}
Therefore, with probability of $1-\text{P}(E_1 \cup E_2) \ge 1-\left(\left(\exp\left(C_1+{C_2}/{2}\right)+8\right)p^{-2}+p^{{-1}/{C_2}}\right)$, neither $E_1$ nor $E_2$ occurs, and hence
\begin{align*}
\norm{\grad F(\bTheta^*)}_\infty \leq & 2(C_3\epsilon_1+\epsilon^2_1)  \sqrt{\frac{\log p}{n}} +2 C_4 \epsilon_2 \\
= & 2 \left[C_3(3\log p + \log n) + (3\log p + \log n)^2 \right]\sqrt{\frac{\log p}{n}} + 2C_4\left( C_1+ \sqrt{\frac{2\log p}{n}} \right). 
\end{align*}

\subsubsection*{Proof of \eqref{eq:bound-der-2}}

Then, we study \eqref{eq:bound-der-2}. We consider a map defined as $G(\bm{\Delta}_S) := - \bH_{SS}^{-1} \left[ \grad_S F(\bTheta^* + \bm{\Delta}_S )+ \lambda \hat{\bZ}_{S} \right] + \bm{\Delta}_{S}$.
If $\norm{\bm{\Delta}}_\infty \leq r$, by Taylor expansion of $\grad_S F(\bTheta^* + \bm{\Delta})$ centered at $\grad_SF(\bTheta^*)$,
\begin{align*}
\Norm{G(\bm{\Delta}_{S})}_\infty \hspace{-2mm} = & \Norm{- \bH_{SS}^{-1} \left[ \grad_S F(\bTheta^* ) + \bH_{SS} \bDelta_S + \bm{R}_S(\bDelta) + \lambda \hat{\bZ}_{S} \right] \hspace{-1mm} + \hspace{-1mm} \bm{\Delta}_{S} }_\infty \hspace{-2mm} = \Norm{ - \bH_{SS}^{-1} \left( \grad_S F(\bTheta^*) + \bm{R}_S(\bm{\Delta}) +\lambda \hat{\bZ}_{S} \right)}_\infty\\
\leq & \Norm{\bH_{SS}^{-1}}_\infty (\norm{\grad_S F(\bTheta^*) }_\infty + \norm{R_S(\bm{\Delta})}_\infty + \lambda \norm{\hat{\bZ}_{S}}_\infty)
\leq  (C_5(\lambda+C_6r^2+\lambda) = C_5C_6r^2 + 2C_5\lambda,
\end{align*}
where the inequality is due to $\norm{\grad_S F(\bTheta^* )}_\infty \leq \lambda$ conditioning on $E_1^c\cap E_2^c$ and according to \eqref{eq:bound-der-1}. Then, based on the definition of $r$, we can derive the upper bound of $\Norm{G(\bm{\Delta}_{S})}_\infty$ as $\Norm{G(\bm{\Delta}_{S})}_\infty \leq r/2 + r/2 = r$.

Therefore, according to the fixed point theorem \citep{ortega2000iterative, yang2011use}, there exists $\bm{\Delta}_S$ satisfying  $G(\bm{\Delta}_S) = \bm{\Delta}_S$, which indicates $\grad_S F(\bTheta^* + \bm{\Delta} ) + \lambda \hat{\bZ}_{S} = \bm{0}$.
Considering that the optimal solution to \eqref{eq:restricted} is unique, 
$\tilde{\bm{\Delta}}_{S} = \bm{\Delta}_S$, whose infinite norm is bounded by $\norm{\tilde{\bm{\Delta}}_S}_\infty \leq r$ , with probability larger than $1-\left(\left(\exp\left(C_1+{C_2}/{2}\right)+8\right)p^{-2}+p^{{-1}/{C_2}}\right)$.
\end{proof}

\begin{lemma}
\label{lem:dual-optimal}
Let $\hat{\bTheta}$ be an optimal solution to \eqref{eq:pre}, and $\hat{\bZ}$ be the corresponding dual solution. If $\hat{\bZ}$ satisfies  $\norm{\hat{\bZ}_{I}}_\infty < 1$, then any given optimal solution to \eqref{eq:pre} $\tilde{\bTheta}$  satisfies $\tilde{\bTheta}_{I} = \bm{0}$. Moreover, if $\bH_{SS}$ is positive definite, then the solution to \eqref{eq:pre} is unique.
\end{lemma}
\begin{proof}
Specifically, following the same rationale as Lemma 1 in \citealt{wainwright2009sharp},  Lemma 1 in \citealt{ravikumar2010high}, and Lemma 2 in \citealt{yang2011use}, we can derive Lemma \ref{lem:dual-optimal} characterizing the optimal solution of \eqref{eq:pre}.
\end{proof}

\subsection{Proof of Theorem~\ref{thm:sparsistency}}
\label{sec:main-proof}
The proof follows the primal-dual witness (PDW) technique, which is widely used in this line of research \citep{wainwright2009sharp, ravikumar2010high, yang2011use, yang2015graphical}. Specifically,  by Lemma \ref{lem:dual-optimal}, we can prove the sparsistency by building an optimal solution to \eqref{eq:pre} satisfying  $\norm{\hat{\bZ}_I}_\infty <1$, which is summarized as \emph{strict dual feasibility} (SDF). To this end, we apply PDW to build a qualified optimal solution with the assumption that $\bH_{SS}$ is positively definite.

\subsubsection*{Solve a Restricted Problem}
First of all, we derive the KKT condition of \eqref{eq:pre}:
\begin{equation}
\label{eq:kkt}
\grad F(\hat{\bTheta}) + \lambda \hat{\bZ} = \bm{0}.
\end{equation}
 
To construct an optimal optimal primal-dual pair solution, we define $\tilde{\bTheta}$ as an optimal solution to the restricted problem:
\begin{align}
\label{eq:restricted}
\begin{split}
\tilde{\bTheta} := \min_{\bTheta}  F(\bTheta)+\lambda \norm{\bTheta}_1,
\end{split}
\end{align}
with $\bTheta_I = \mathbf{0}$, where $\tilde{\bTheta}$ is unique according to Lemma \ref{lem:dual-optimal} with the assumption that $\bH_{SS} \succ \mathbf{0}$.  Denote the subgradient  corresponding to $\tilde{\bTheta}$  as $\tilde{\bZ}$. Then $(\tilde{\bTheta},\tilde{\bZ})$ is optimal for the restricted problem \eqref{eq:restricted}. Therefore, $\tilde{\bZ}_S$ can be determined according to the values of $\tilde{\bTheta}_S$ via the KKT conditions of \eqref{eq:restricted}. As a result,
\begin{equation}
\label{eq:kkt-s}
\grad_S F(\tilde{\bTheta}) + \lambda \tilde{\bZ}_S = \bm{0},
\end{equation}
where $\grad_S$ represents the gradient components with respect to $S$. Furthermore, by letting $\hat{\bTheta} = \tilde{\bTheta}$, we determine $\tilde{\bZ}_I$ according to \eqref{eq:kkt}. It remains to show that $\tilde{\bZ}_I$ satisfies SDF.

\subsubsection*{Check SDF}
Now, we demonstrate that $\tilde{\bTheta}$ and $\tilde{\bZ}$ satisfy SDF. By \eqref{eq:kkt-s}, and by the Taylor expansion of $\grad_S F(\tilde{\bTheta})$, we have that 
\begin{equation}
\label{eq:tilde-delta-s}
\bH_{SS} \tilde{\bm{\Delta}}_{S} + \grad_S F(\bTheta^*) + \bm{R}_S(\tilde{\bm{\Delta}}) +\lambda \tilde{\bZ}_{S} = \bm{0} \Rightarrow \tilde{\bm{\Delta}}_{S} =\bH_{SS}^{-1} \left[-\grad_S F(\bTheta^*) - \bm{R}_S(\bm{\tilde{\Delta}}) - \lambda \tilde{\bZ}_{S}\right],
\end{equation}
where $\tilde{\bm{\Delta}} := \tilde{\bTheta} - \bTheta^*$, $\bm{R}_S(\tilde{\bDelta})$ represents the components of $\bm{R}(\bDelta)$ corresponding to $S$, and we have used the fact that $\bH_{SS}$ is positive definite and hence invertible. By the definition of $\tilde{\bTheta}$ and $\tilde{\bZ}$,
\begin{equation}
\label{eq:delta}
\grad F(\tilde{\bTheta}) + \lambda \tilde{\bZ} = \mathbf{0} \Rightarrow \grad F(\bTheta^*) + \bH \tilde{\bDelta} + \bm{R}(\tilde{\bDelta}) + \lambda\tilde{\bZ} = \bm{0} \Rightarrow \grad_I F(\tilde{\bTheta})  + \bH_{IS} \tilde{\bDelta}_S+ \bm{R}_I(\tilde{\bDelta}) + \lambda \tilde{\bZ}_I=\bm{0},
\end{equation}
where $\bm{R}_I(\tilde{\bDelta})$ represents the components of $\bm{R}(\bDelta)$ corresponding to $I$,  and we have used the fact that $\tilde{\bDelta}_I=\bm{0}$ because $\tilde{\bTheta}_I=\bTheta^*=\mathbf{0}$. As a result, 
\begin{align}
\label{eq:tilde-z-i-1}
\lambda \norm{\tilde{\bZ}_{I}}_\infty 
= & \norm{-\bH_{IS} \tilde{\bm{\Delta}}_{S}- \grad_I F(\bTheta^*) -\bm{R}_I(\tilde{\bm{\Delta}})}_\infty \nonumber \\
\leq &\Norm{\bH_{IS}   \bH_{SS}^{-1} \left[-\grad_S F(\bTheta^*) - \bm{R}_S(\tilde{\bm{\Delta}} )- \lambda \tilde{\bZ}_{S}\right]}_\infty + \norm{\grad_I F(\bTheta^*) + \bm{R}_I(\tilde{\bm{\Delta}})}_\infty
\nonumber \\ 
\le & \Norm{\bH_{IS}   \bH_{SS}^{-1} }_{\infty} \Norm{\grad_S F(\bTheta^*) + \bm{R}_S(\tilde{\bm{\Delta}})}_{\infty} + \Norm{\bH_{IS} \bH_{SS}^{-1} }_{\infty} \Norm{\lambda \tilde{\bZ}_{S}}_{\infty} + \norm{\grad_I F(\bTheta^*) + \bm{R}_I(\tilde{\bm{\Delta}})}_\infty \nonumber \\
\le & (1-\alpha) \left(  \norm{\grad_S F(\bTheta^*)}_\infty+\norm{\bm{R}_S(\tilde{\bm{\Delta}})}_\infty  \right) + (1-\alpha)\lambda + \left(  \norm{\grad_I F(\bTheta^*)}_\infty + \norm{\bm{R}_I (\tilde{\bm{\Delta}})}_\infty  \right) \nonumber \\
\leq & (2-\alpha) \left(  \norm{\grad F(\bTheta^*)}_\infty+\norm{\bm{R}(\tilde{\bm{\Delta}})}_\infty  \right) + (1-\alpha)\lambda,
\end{align}
where we have used \eqref{eq:tilde-delta-s} in the first inequality, and the third inequality is due to Assumption \ref{asm:3}. 

With \eqref{eq:tilde-z-i-1}, it remains to control $\norm{\grad F(\bTheta^*)}_\infty$ and $\norm{\bm{R}(\tilde{\bm{\Delta}})}_\infty$. On one hand, according to Lemma \ref{lem:bound-der} and the assumption on $\lambda$ in Theorem~\ref{thm:sparsistency},  $\norm{\grad F(\bTheta^*)}_\infty \leq 2 \left[3C_3\log p + C_3\log n + (3\log p + \log n)^2 \right]\sqrt{\frac{\log p}{n}} + 2C_4\left( C_1+ \sqrt{\frac{2\log p}{n}} \right)\leq \frac{\alpha \lambda}{4} $, with probability larger than $1-\left(\left(\exp\left(C_1+{C_2}/{2}\right)+8\right)p^{-2}+p^{{-1}/{C_2}}\right)$.  

On the other hand, according to Assumption~\ref{asm:4} and Lemma~\ref{lem:bound-der}, 
\begin{equation}
\label{eq:r-tilde-delta-inf}
\norm{\bm{R}(\tilde{\bm{\Delta}})}_\infty
 \leq C_6 \norm{\bDelta}_\infty^2 \leq  C_6 r^2 \leq C_6 (4C_5\lambda)^2 = \lambda \frac{64C_5^2 C_6}{\alpha} \frac{\alpha \lambda}{4} \le \left(C_7 \sqrt{\frac{\log^5 p}{n}}\right) \frac{64C_5^2 C_6}{\alpha}  \frac{\alpha \lambda}{4},
\end{equation}
where in the last inequality we have used the assumption $\lambda \propto \sqrt{\frac{\log^5 p}{n}}$ in Theorem~\ref{thm:sparsistency}, and hence there exists $C_7$ satisfying  $\lambda \leq  C_7 \sqrt{\frac{\log^5 p}{n}}$. Therefore, when we choose $n \ge \left(64C_7  C_5^2C_6/\alpha\right)^2\log^5 p$ as assumed in Theorem~1, then from \eqref{eq:r-tilde-delta-inf}, we can conclude that $\norm{\bm{R}(\tilde{\bm{\Delta}})}_\infty \leq \frac{\alpha \lambda }{4}$. As a result, $\lambda \norm{\hat{\bZ}_{I}}_\infty $ can be bounded by $\lambda \norm{\tilde{\bZ}_{I}}_\infty   < \alpha\lambda/2+\alpha\lambda/2 +(1-\alpha)\lambda = \lambda$.
Combined with Lemma \ref{lem:dual-optimal}, we demonstrate that any optimal solution of \eqref{eq:pre} satisfies $\tilde{\bTheta}_{I} = \bm{0}$. Furthermore, \eqref{eq:bound-der-2} controls the difference between the optimal solution of \eqref{eq:pre} and the real parameter by $\norm{\tilde{\bm{\Delta}}_S}_\infty \leq r$, by the fact that $r \le \norm{\bTheta^*_S}_\infty$ in Theorem~\ref{thm:sparsistency}, $\hat{\bTheta}_S$ shares the same sign with $\bTheta^*_S$. 

\subsection{Details of the ADR Experiments}

In this section, we provide more details regarding the parameter selection of the proposed method in the ADR application discussed in Section~\ref{sec:experiment}. More specifically, to achieve better performance, several empirical design chocies are made.

\subsubsection{Future Effects Discount}
To start with, we formulate the Poisson pseudo-likelihood for the ADR application. 
Assume patient $i$ with $\bt_i := \begin{bmatrix}
t_{i1} \ t_{i2} \ \cdots \ t_{in_i}
\end{bmatrix}^\top$, $\bo_i := \begin{bmatrix}
o_{i1} \ o_{i2} \ \cdots \ o_{in_i}
\end{bmatrix}^\top$, and $\bx_i := \begin{bmatrix}
x_{i1} & x_{i2} & \cdots & x_{in_i}
\end{bmatrix}^\top$.
Following the analysis in Section~\ref{sec:estimation}, the Poisson pseudo-likelihood method is equivalent to $n_i$ Poisson regressions. For the $j^{\text{th}}$ regression ($j \in [n_i]$), the response is $\bx_{ij}$, the covariates are 
\begin{equation*}
    \bphi(\abs{t_{i1} -t_{ij}})x_{i1},\bphi(\abs{t_{i2} -t_{ij}})x_{i2}, \cdots, \bphi(\abs{t_{ij-1} -t_{ij}})x_{ij-1}, \bphi(\abs{t_{ij+1} -t_{ij}})x_{ij+1},\cdots, \bphi(\abs{t_{i1} -t_{in_i}})x_{in_i},
\end{equation*} 
and the parameters are $\bomega$ and $\bw$ defined in \eqref{eq:omega-i}. 
As a result, for $N$ observed patients, the maximal Poisson pseudo-likelihood estimator proposed is equivalent to the solution to the $\sum_{i=1}^N n_i$ regressions formulated above.

When implementing the proposed Poisson pseudo-likelihood method for the ADR application, the performance can benefit from penalizing the effect from future events to previous ones, denoted by $\bphi(\abs{t_{ij'} -t_{ij}})x_{ij'}$, with $t_{ij'}>t_{ij}$. Therefore, we include two hyperparameters $\lambda_1$ and $\lambda_2$ to discount the effects of $\bphi(\abs{t_{ij'} -t_{ij}})x_{ij'}$. Specifically, the discount to the future effects is two fold: 
\begin{itemize}
    \item The time-threshold for the influence function is multiplied by $\lambda_1$, which is equivalent to assuming that the future events have shorter effects to previous ones. 
    \item $\bphi(\abs{t_{ij'} -t_{ij}})x_{ij'}$ is directly multiplied by $\lambda_2$, to decrease the effects. 
\end{itemize}

\subsubsection{Data Pre-Processing for $\bx$}
In the EHR dataset for the ADR application, $\bx$ contains lots of $0$ components, since consequent repeated medical events are often not recorded in the EHR dataset. The effects of events recorded only once in the dataset are not properly addressed when the corresponding covariates are directly defined as $0$, following the data aggregation procedure in Section~\ref{sec:data-aggreation}. To deal with this, we add 1 to all the defined $\bx$, which omits the $0$ components without changing the relative order of the components of $\bx$. Also, it should be noticed that this phenomenon is especially significant in EHR datasets, and thus may not be generalized to other applications.  

\subsubsection{Occurrence Time Ambiguity}
Another challenge caused by the EHR data is the occurrence time ambiguity: the occurrence time in the EHR data is often not accurate, which especially causes problems in the data aggregation process in Section~\ref{sec:data-aggreation}. Specifically, the data aggregation method proposed relies on the occurrence order of events to be able to determine consecutive occurrences, and to aggregate them to one group. Due to the inaccurate occurrence time of in EHR, the proposed aggregation tends to miss to aggregate consecutive occurrences.  

To solve this problem, we introduce the time ambiguity parameter $T_{ambiguity}$ denoting the time ambiguity the user choose tolerate when aggregating same-type observations. Specifically, if we observe a sequence of events like $\curly{A, B, A, A, A}$ and the occurrence time between the first $A$ and $B$ is less than $T_{ambiguity}$, we will still treat the observed $A$'s as consecutive occurrences and aggregate them to the first $A$. In other words, we cannot determine whether the occurrence order of the first $A$ and $B$ is true, since their occurrence time is too close. We choose to still conduct the aggregation and the tolerate this as a occurrence time ambiguity. 

\subsubsection{Extra Experiment Results}
With the aforementioned techniques, we select the future discount as $\lambda_1 = \lambda_2 = 0.1$, and the time ambiguity parameter as $T_{ambiguity} = 175$. Also, for the data pre-processing, we choose the minimum duration as $1000$, and the maximum time difference as $1500$ (with the unit as one third of a day). We define the influence function with $L=3$. Then, the AUC achieved by the regularization parameters selected by AIC reaches $0.8965$. Note that although further optimizing the hyper parameters can lead to even higher AUCs, such results may be under the risk of over fitting. 
Developing better hyperparameter selection procedures will be an interesting direction for future work.